\documentclass[english,final]{siamltex}
\usepackage[T1]{fontenc}
\usepackage[latin9]{inputenc}
\usepackage{amssymb}
\usepackage{graphicx}
\usepackage{esint}

\makeatletter

\providecommand{\tabularnewline}{\\}

\newcommand\eqref[1]{(\ref{#1})}

\usepackage{babel}

\makeatother

\title{Distance Transform Gradient Density Estimation using the Stationary
Phase Approximation}

\author{Karthik S. Gurumoorthy\footnotemark[3] 
\thanks{Email: {\tt sgk@ufl.edu}}
\and Anand Rangarajan\footnotemark[3]
\thanks{This work is partially supported by NSF IIS 1143963. Email: {\tt anand@cise.ufl.edu}
}}
\begin{document}
\maketitle
\renewcommand{\thefootnote}{\fnsymbol{footnote}}
\footnotetext[3]{Department of Computer and Information Science and Engineering,
University of Florida, Gainesville, Florida, USA}

\begin{abstract}
The complex wave representation (CWR) converts unsigned 2D distance
transforms into their corresponding wave functions. Here, the distance
transform $S(X)$ appears as the phase of the wave function $\phi(X)$---specifically,
$\phi(X)=\exp\left(\frac{iS(X)}{\tau}\right)$ where $\tau$ is a
free parameter. In this work, we prove a novel result using the higher-order
stationary phase approximation: we show convergence of the normalized
power spectrum (squared magnitude of the Fourier transform) of the
wave function to the density function of the distance transform gradients
as the free parameter $\tau\rightarrow0$. In colloquial terms, \emph{spatial
frequencies are gradient histogram bins}. Since distance transform
gradients carry only orientation information (as their magnitudes are
identically equal to one almost everywhere), 
the 2D Fourier transform values mainly lie on the unit circle in the
spatial frequency domain as $\tau\rightarrow0$. The proof of the result
involves standard 
integration techniques and requires proper ordering of limits. Our
mathematical relation indicates that the CWR of distance transforms
is an intriguing, new representation. \end{abstract}

\pagestyle{myheadings}
\thispagestyle{plain}
\markboth{KARTHIK S. GURUMOORTHY AND ANAND RANGARAJAN}{DISTANCE TRANSFORM DENSITY ESTIMATION}

\section{Introduction}

Euclidean distance functions (more popularly referred
to as distance transforms) are widely used in many domains \cite{Osher02,Siddiqi08}.
An important subset---point-set based distance functions---also finds application in many domains
with computer vision being a prominent example \cite{Siddiqi08,Gurumoorthy09,Sethi12}. 
Since distance transforms allow us to transition
from shapes to a scalar field, problems such as shape
registration are often couched in terms of rigid, affine or nonrigid
alignment of distance transform fields, where the shapes are parameterized as a set of points \cite{Paragios03}. 
In medical imaging, they are used in the construction of neuroanatomical shape complex atlases based on an information geometry framework \cite{Chen10}.

Even when one begins with a set of closed
curves (as a shape template for example), the curves are often discretized
to yield a point-set prior to the application of fast sweeping \cite{Zhao05} and other
distance transform estimation methods. Signed and unsigned distance transforms are deployed in 3D as well
with their zero level-sets corresponding to surfaces. Furthermore,
medial axis methods and skeletonization often involve distance transform representations \cite{Kimmel03}.
In the domain of computer vision, the gradient density function is popularly
known as the histogram of oriented gradients (HOG). 
Since the advent of HOG a few years ago, gradient density estimation has risen
in prominence and is employed in human recognition systems \cite{Dalal05}.

The distance transform for a set of $K$ discrete
points $Y=\{Y_{k}\in\mathbb{R}^{D}\},k\in\{1,\ldots,K\}$ where $D$
is the dimensionality of the point-set is defined as 
\begin{equation}
S(X)\equiv\min_{k}\|X-Y_{k}\|,\label{eq:EucdistProb}
\end{equation}
 where $X\in\Omega$ is a closed bounded domain in $\mathbb{R}^{D}$.
In this article, we are only concerned with $D=2$.

In computational geometry, Euclidean distance functions correspond to the Voronoi problem \cite{deBerg08} and the solution
$S(X)$ can be visualized as a set of cones (with the centers being
the point-set locations $\{Y_{k}\}$). The distance transform satisfies
the static, non-linear Hamilton-Jacobi equation 
\begin{equation}
\|\nabla S\|=1\label{eq:gradNorm}
\end{equation}
 almost everywhere, barring the point-set locations and the Voronoi
boundaries where it is not differentiable \cite{Osher02,Osher88,Sethian96}.
Here $\nabla S=(S_{x},S_{y})$ denotes the gradients of
$S$ and $\|\cdot\|$ represents its Euclidean magnitude. Furthermore
$S(X)=0$ at the point-set locations. Following the wave optics literature,
one can envisage light waves simultaneously emanating from the given
point sources and propagating with a velocity of one in all directions.
The value of $S$ at a grid point $X_{0}$, namely $S(X_{0})$, corresponds
to the \emph{time taken} by the first light wave (out of the $K$
light waves) to reach the grid location $X_{0}$.  Driven by this optics
analogy, when we express $S$ as the phase of a wave function $\phi$
as in 
\begin{equation}
\phi=\exp\left(\frac{iS}{\tau}\right),\label{eq:phaseRelation}
\end{equation}
we made an intriguing empirical observation. The power spectrum of the wave
function approximates the density function of the gradients of the distance
transform as  
the parameter $\tau$ in Equation~\ref{eq:phaseRelation} tends to zero. In this paper, we formally prove this result.
We refer to this wave function $\phi$ which satisfies the phase relation
with $S$ as the \emph{Complex Wave Representation} (CWR) of distance transforms.

\section{Main Contribution}
The centerpiece of this work is to provide a useful application of
the stationary phase method, wherein we show an equivalence between
the density function of the gradients of the distance function $\nabla S=(S_{x},S_{y})$
and the power spectrum (squared magnitude of the Fourier transform)
of the CWR ($\phi$) as the free parameter $\tau$ (in Equation~\ref{eq:phaseRelation})
approaches zero. Here, the density function of the gradients is obtained
via a random variable transformation of a uniformly distributed random
variable $W$ (over the bounded domain $\Omega$) using the gradients
$\nabla S=(S_{x},S_{y})$ as the transformation functions.
In other words, if we define a random variable $Z=\nabla S(W)$
where the random variable $W$ has a \emph{uniform distribution} on
a closed bounded domain $\Omega\subset\mathbb{R}^{2}$, the density
function of $Z$ represents the density function of the gradients
of the distance transform.

As the norm of the gradients $\nabla S$ is defined to be
$1$ almost everywhere (from Equation~\ref{eq:gradNorm}), we observe
that the density function of the gradients is one-dimensional and
defined over the space of orientations. Section~\ref{sec:densityfunction}
provides a closed-form expression for this density function. As the
gradients are unit vectors, we notice that the Fourier transform values
of the CWR ($\phi$) lie mainly on the unit circle and this behavior
tightens as $\tau\rightarrow0$. Specifically, if $F_{\tau}(\tilde{r},\omega)$
represents the Fourier transform of $\phi$ in the polar coordinate
system at a given value of $\tau$, Theorem~\ref{CircleTheorem}
demonstrates that if $\tilde{r}\not=1$, then $\lim_{\tau\rightarrow0}F_{\tau}(\tilde{r},\omega)=0$.

Our main result is established in Theorem~\ref{IntegralLimitTheorem}
where we show that the power spectrum of the wave function $\phi$
when polled close to the unit circle, is approximately equal to the
density function of the distance transform gradients, with the approximation
becoming increasingly exact as $\tau\rightarrow0$. In other words,
if $P(\omega)$ denotes the closed-form density of the gradients defined
over the orientation $\omega$ and if $P_{\tau}(\tilde{r},\omega)$
corresponds to the power spectrum of $\phi$ represented in the polar
coordinate system at a given value of $\tau$, Theorem~\ref{IntegralLimitTheorem}
constitutes the following relation 
\begin{equation}
\lim_{\delta\rightarrow0}\lim_{\tau\rightarrow0}\int_{\omega_{0}}^{\omega_{0}+\Delta}\left\{ \int_{1-\delta}^{1+\delta}P_{\tau}(\tilde{r},\omega)\tilde{r}d\tilde{r}\right\} d\omega=\int_{\omega_{0}}^{\omega_{0}+\Delta}P(\omega)d\omega
\end{equation}
 for any (small) value of the interval measure $\Delta$ on $\omega$.
We show this result using the higher-order \emph{stationary
phase approximation}, a well known technique in asymptotic analysis
\cite{Wong89}. Through the pioneering works of Jones and Kline \cite{Jones58},
Olver \cite{OlverBook74}, Wong \cite{Wong89}, McClure and Wong \cite{McClure91},
among others, the stationary phase approximation has become a widely
deployed tool in the approximation of oscillatory integrals. Our work
showcases a novel application of the stationary phase method for estimating
the probability density function of distance transform gradients.
The significance of our mathematical result is that \emph{spatial
frequencies become histogram bins} and hence the power spectrum $P_{\tau}$
can serve as a gradient density estimator at small, non-zero values
of $\tau$. We would like to emphasize that our work is \emph{fundamentally
different} from estimating the gradients of a density function \cite{Fukunaga75}
and should not be semantically confused with it.

\subsection{Motivation from quantum mechanics}

Our new mathematical relationship is motivated by the classical-quantum
relation, wherein classical physics is expressed as a limiting case
of quantum mechanics \cite{Griffiths04,Feynman2010}. When $S$ is treated
as the Hamilton-Jacobi scalar field, the gradients of $S$ correspond
to the classical momentum of a particle \cite{Goldstein02}. In the
parlance of quantum mechanics, the squared magnitude of the wave function
expressed either in its position or momentum basis corresponds to
its position or momentum density respectively. Since these representations
(either in the position or momentum basis) are simply (suitably scaled)
Fourier transforms of each other, the squared magnitude of the Fourier
transform of the wave function expressed in its position basis is
its quantum momentum density. However, the time independent Schr\"odinger
wave function $\phi(x,y)$ (expressed in its position basis) can be
approximated by $\exp\left(\frac{iS(x,y)}{\tau}\right)$ as $\tau\rightarrow0$
\cite{Feynman2010}. Here $\tau$ (treated as a free parameter in our work)
represents 
Planck's constant. Hence the squared magnitude of the Fourier transform
$\exp\left(\frac{iS(x,y)}{\tau}\right)$ corresponds to the quantum
momentum density of $S$. The principal results proved in the article
(Theorem~\ref{IntegralLimitTheorem} and Proposition~\ref{IntegrationOverrtilde})
state that the classical momentum density (denoted by $P$) can be
expressed as a limiting case (as $\tau\rightarrow0$) of its corresponding
quantum momentum density (denoted by $P_{\tau}$), in agreement with
the correspondence principle.

\section{The Distance Transform Gradient Density Function\label{sec:densityfunction}}

As mentioned above, the geometry of the distance transform corresponds
to a set of intersecting cones with the origins at the Voronoi centers
\cite{deBerg08}. The gradients of the distance transform (which exist
globally except at the cone intersections and origins) are unit vectors
and satisfy Equation~\ref{eq:gradNorm}. Therefore the gradient density
function is one-dimensional and defined over the space of \emph{orientations}.
The orientations are constant and unique along each ray of each cone.
Its probability distribution function is given by 
\begin{equation}
\mathcal{F}(\theta\leq\Theta\leq\theta+\Delta)\equiv\frac{1}{L}\int\int_{\theta\leq\arctan\left(\frac{S_{y}}{S_{x}}\right)\leq\theta+\Delta}dxdy\label{def:distributionfunc}
\end{equation}
where $L$ is the area of the bounded domain $\Omega$. We have expressed the orientation random variable as $\Theta=\arctan\left(\frac{S_{y}}{S_{x}}\right)$.
The probability distribution function also induces a \emph{closed-form
expression} for its density function as shown below.

Let $\Omega\subset\mathbb{R}^{2}$ denote a polygonal grid such that
its boundary $\partial\Omega$ is composed of a finite sequence of
straight line segments. The reason for restricting only to polygonal
domains with boundaries made of line segments will become clear when
we discuss Theorem~\ref{CircleTheorem}. Let the set $Y=\{Y_{k}\in\mathbb{R}^{2},k\in\{1,\ldots,K\}\}$
be the given point-set locations and let $Y_{k}=(x_{k},y_{k})$. Then
the Euclidean distance transform at a point $X=(x,y)\in\Omega$ is
given by 
\begin{equation}
S(X)\equiv\min_{k}\|X-Y_{k}\|=\min_{k}(\sqrt{(x-x_{k})^{2}+(y-y_{k})^{2}}).\label{def:distanceProblem}
\end{equation}
 Let $\mathcal{D}_{k}$, \emph{centered} at $Y_{k}$, denote the $k^{th}$
Voronoi region corresponding to the input point $Y_{k}$. $\mathcal{D}_{k}$
can be represented by a Cartesian product $[0,2\pi)\times[0,R_{k}(\theta)]$
where $R_{k}(\theta)$ is the length of the ray of the $k^{th}$ cone
at an orientation $\theta$. If a grid point $X=(x,y)\in\left(Y_{k}+\mathcal{D}_{k}\right)$,
then $S(X)=\|X-Y_{k}\|$. Each $\mathcal{D}_{k}$ is a convex polygon
whose boundary $\partial\mathcal{D}_{k}$ is also composed of a finite
sequence of straight line segments as shown in Figure~\ref{fig:Voronoi}.

\noindent \begin{center}
\begin{figure}[ht!]
\begin{centering}
\includegraphics[width=0.5\textwidth]{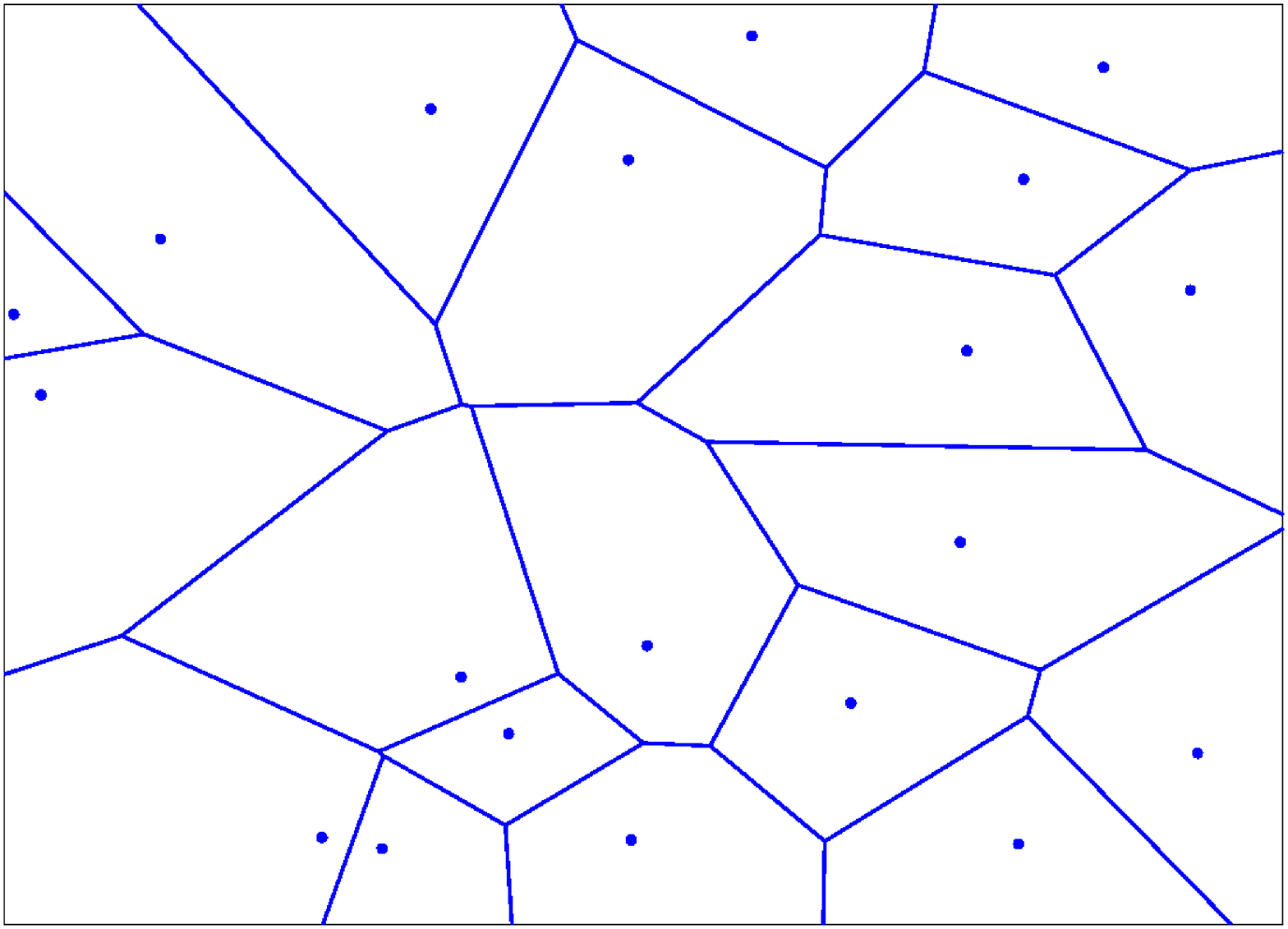} 
\par\end{centering}

\caption{Voronoi diagram of the given $K$ points. Each Voronoi boundary is
composed of straight line segments.}

\label{fig:Voronoi} 
\end{figure}

\par\end{center}

Note that even for points that lie on the Voronoi boundary where the
radial length equals $R_{k}(\theta)$, the distance transform is well
defined. The area $L$ of the polygonal grid $\Omega$ is given by
\begin{equation}
L\equiv\sum_{k=1}^{K}\int_{0}^{2\pi}\int_{0}^{R_{k}(\theta)}rdrd\theta=\sum_{k=1}^{K}\int_{0}^{2\pi}\frac{R_{k}^{2}(\theta)}{2}d\theta.\label{eq:L}
\end{equation}
 With the above set-up in place, after recognizing the cone geometry
at each Voronoi center $Y_{k}$, Equation~\ref{def:distributionfunc}
can be simplified as 
\begin{equation}
\mathcal{F}(\theta\leq\Theta\leq\theta+\Delta)\equiv\frac{1}{L}\sum_{k=1}^{K}\int_{\theta}^{\theta+\Delta}\int_{0}^{R_{k}(\theta)}rdrd\theta=\frac{1}{L}\sum_{k=1}^{K}\int_{\theta}^{\theta+\Delta}\frac{R_{k}^{2}(\theta)}{2}d\theta.
\end{equation}
 Following this drastic simplification, we can write the closed-form
expression for the density function of the unit vector distance transform
gradients as 
\begin{equation}
P(\theta)\equiv\lim_{\Delta\rightarrow0}\frac{\mathcal{F}(\theta\leq\Theta\leq\theta+\Delta)}{\Delta}=\frac{1}{L}\sum_{k=1}^{K}\frac{R_{k}^{2}(\theta)}{2}.\label{def:densityfunc}
\end{equation}
 Based on the expression for $L$ in Equation~\ref{eq:L}, it is easy
to see that 
\begin{equation}
\int_{0}^{2\pi}P(\theta)d\theta=1.\label{eq:PInt}
\end{equation}
 Since the Voronoi cells are convex polygons \cite{deBerg08}, each
cell contributes exactly one conical ray to the density function on
orientation.

\section{Properties of the Fourier Transform of the CWR\label{sec:FTProperties}}

Since the distance transform is not differentiable at the point-set
locations $\{Y_{k}\}_{k=1}^{K}$ and also along the Voronoi boundaries
$\partial\mathcal{D}_{k},\forall k$ (a measure zero set in 2D), we
restrict ourselves to the region which excludes both of them. To this
end, let $0<\epsilon<\frac{1}{2}$ be given. Let the region $\mathcal{D}_{k}^{\epsilon}$
centered at $Y_{k}$ be represented by the Cartesian product $[0,2\pi)\times[R_{k}^{(1)}(\theta),R_{k}^{(2)}(\theta)]$
where, 
\begin{eqnarray}
R_{k}^{(1)}(\theta) & = & \epsilon R_{k}(\theta)\hspace{10pt}\mbox{and}\nonumber \\
R_{k}^{(2)}(\theta) & = & (1-\epsilon)R_{k}(\theta).
\end{eqnarray}
 The length of the ray at the orientation $\theta$ in $\mathcal{D}_{k}^{\epsilon}$
equals $R_{k}^{(2)}(\theta)-R_{k}^{(1)}(\theta)$. Note that in the
definition of $\mathcal{D}_{k}^{\epsilon}$, we have explicitly removed
the source point $Y_{k}$ where the ray length $r(\theta)=0$ and
the boundary of the Voronoi cell where $r(\theta)=R_{k}(\theta)$
as shown in Figure~\ref{fig:ExcludeVBoun}. 

\noindent \begin{center}
\begin{figure}[ht!]
\begin{centering}
\includegraphics[width=0.4\textwidth]{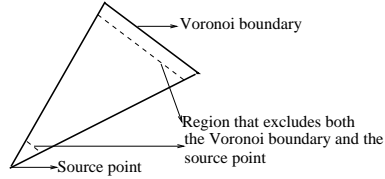} 
\par\end{centering}

\caption{Region that excludes both the source point and the Voronoi boundary.}

\label{fig:ExcludeVBoun} 
\end{figure}

\par\end{center}

\noindent Define the grid 
\begin{equation}
\Omega^{\epsilon}\equiv\bigcup_{k=1}^{K}\left(Y_{k}+\mathcal{D}_{k}^{\epsilon}\right).
\end{equation}
 Its area $L^{\epsilon}$ equals 
\begin{equation}
L^{\epsilon}\equiv\sum_{k=1}^{K}\int_{0}^{2\pi}\int_{R_{k}^{(1)}(\theta)}^{R_{k}^{(2)}(\theta)}rdrd\theta=(1-2\epsilon)\sum_{k=1}^{K}\int_{0}^{2\pi}\frac{R_{k}^{2}(\theta)}{2}d\theta.
\end{equation}
 From Equation~\ref{eq:L} we get $L^{\epsilon}=(1-2\epsilon)L$
and hence $\lim_{\epsilon\rightarrow0}L^{\epsilon}=L$.

\noindent Let $l^{\epsilon}=\sqrt{L^{\epsilon}}$. Define a function
$F^{\epsilon}:\mathbb{R}\times\mathbb{R}\times\mathbb{R}\rightarrow\mathbb{C}$
as 
\begin{equation}
F^{\epsilon}(u,v,\tau)\equiv\frac{1}{2\pi\tau l^{\epsilon}}\iint\limits _{\Omega^{\epsilon}}\exp\left(\frac{iS(x,y)}{\tau}\right)\exp\left(\frac{-i(ux+vy)}{\tau}\right)dxdy.\label{eq:Fuvhbar}
\end{equation}
 For a fixed value of $\tau$, define a function $F_{\tau}^{\epsilon}:\mathbb{R}\times\mathbb{R}\rightarrow\mathbb{C}$
as 
\begin{equation}
F_{\tau}^{\epsilon}(u,v)\equiv F^{\epsilon}(u,v,\tau).\label{eq:Fhbar}
\end{equation}
 Note that $F_{\tau}^{\epsilon}$ is closely related to the Fourier
transform of the CWR, $\phi=\exp\left(\frac{iS}{\tau}\right)$ \cite{Bracewell99}.
The scale factor $\frac{1}{2\pi\tau l^{\epsilon}}$ is the normalization
factor such that the $\ell_{2}$ norm of $F_{\tau}^{\epsilon}$ is 1 as
seen in the following Lemma (with the proof given in Appendix~\ref{sec:proofIntLemma}). 
\begin{lemma}
\label{IntegralLemma} With $F_{\tau}^{\epsilon}$ defined as above,
$F_{\tau}^{\epsilon}\in L^{2}(\mathbb{R}^{2})$ and $\|F_{\tau}^{\epsilon}\|=1$.
\\

\end{lemma}
Consider the polar representation of the spatial frequencies $(u,v)$
namely $u=\tilde{r}\cos(\omega)$ and $v=\tilde{r}\sin(\omega)$ where
$\tilde{r}>0$. For $(x,y)\in\left(Y_{k}+\mathcal{D}_{k}^{\epsilon}\right)$,
let $x-x_{k}=r\cos(\theta)$ and $y-y_{k}=r\sin(\theta)$ where $r\in[R_{k}^{(1)}(\theta),R_{k}^{(2)}(\theta)]$.
Then Equation~\ref{eq:Fuvhbar} can be rewritten as 
\begin{equation}
F_{\tau}^{\epsilon}(\tilde{r},\omega)=\sum_{k=1}^{K}C_{k}I_{k}(\tilde{r},\omega)
\end{equation}
 where 
\begin{equation}
C_{k}=\exp\left\{ -\frac{i}{\tau}\left[\tilde{r}\cos(\omega)x_{k}+\tilde{r}\sin(\omega)y_{k}\right]\right\} 
\end{equation}
 and 
\begin{equation}
I_{k}(\tilde{r},\omega)=\frac{1}{2\pi\tau l^{\epsilon}}\int_{0}^{2\pi}\int_{R_{k}^{(1)}(\theta)}^{R_{k}^{(2)}(\theta)}\exp\left\{ \frac{i}{\tau}r\left[1-\tilde{r}\cos(\theta-\omega)\right]\right\} rdrd\theta.\label{def:Ik}
\end{equation}

With the above set-up in place, we have the following theorem, namely, 
\begin{theorem}
{[}Circle Theorem{]} \label{CircleTheorem} If $\tilde{r}\not=1$,
then, 
\begin{equation}
\lim_{\tau\rightarrow0}F_{\tau}^{\epsilon}(\tilde{r},\omega)=0,
\end{equation}
 for any $0<\epsilon<\frac{1}{2}$. 
\end{theorem}

\subsection{An Intuitive Examination of Theorem~\ref{CircleTheorem}}

Before we furnish a rigorous proof for the aforementioned theorem,
we provide an intuitive picture of why the statement is true. Observe
that the first exponential $\exp\left(\frac{iS(x,y)}{\tau}\right)$
in Equation~\ref{eq:Fuvhbar} is a varying complex \textquotedbl{}sinusoid\textquotedbl{}
and the second exponential $\exp\left(\frac{-i(ux+vy)}{\tau}\right)$
in Equation~\ref{eq:Fuvhbar} is a fixed complex sinusoid at frequencies
$\frac{u}{\tau}$ and $\frac{v}{\tau}$ along the $x$- and $y$-coordinate axes respectively. When we multiply
these two complex exponentials, at low values of $\tau$, the two
sinusoids are usually not \textquotedbl{}in sync\textquotedbl{} and
cancellations occur in the integral. Exceptions to the cancellation
happen at locations where $\nabla S=(S_{x},S_{y})=(u,v)$,
as around these locations, the two sinusoids are in perfect sync.
Since $\|\nabla S\|=1$ for distance transforms, strong
resonance occurs only when $u^{2}+v^{2}=1$ ($\tilde{r}=1$). When
$\tilde{r}\not=1$, the two sinusoids tend to \emph{cancel} each other
out as $\tau\rightarrow0$, resulting in $F_{\tau}^{\epsilon}$ becoming
zero at these locations.

\subsection{Proof of Theorem~\ref{CircleTheorem}}

Having given an intuitive picture of why Theorem~\ref{CircleTheorem}
holds true, we now proceed with the formal proof. As each $C_{k}$
is bounded, it suffices to show that if $\tilde{r}\not=1$, then $\lim_{\tau\rightarrow0}I_{k}(\tilde{r},\omega)=0$
for all $I_{k}$. 
\begin{proof}
\noindent Consider the integral 
\begin{equation}
I(\tilde{r},\omega)=\frac{1}{2\pi\tau l^{\epsilon}}\int_{0}^{2\pi}\int_{R^{(1)}(\theta)}^{R^{(2)}(\theta)}\exp\left\{ \frac{i}{\tau}r\left[1-\tilde{r}\cos(\theta-\omega)\right]\right\} rdrd\theta,\label{def:I}
\end{equation}
 where $R^{(1)}(\theta)=\epsilon R(\theta)$ and $R^{(2)}(\theta)=(1-\epsilon)R(\theta)$.
Let the region $[0,2\pi)\times[R^{(1)}(\theta),R^{(2)}(\theta)]$
be denoted by $\mathcal{D}^{\epsilon}$. $R(\theta)$ is defined in
such a way that the boundary of $D^{\epsilon}$ consists of a finite
sequence of straight line segments as in the case of each $\mathcal{D}_{k}^{\epsilon}$.
Notice that $\mathcal{D}^{\epsilon}$ doesn't contain the origin $(0,0)$.
In order to prove Theorem~\ref{CircleTheorem}, it is sufficient
to show that $\lim_{\tau\rightarrow0}I(\tilde{r},\omega)=0$.

Let $p(r,\theta;\tilde{r},\omega)=r(1-\tilde{r}\cos(\theta-\omega))$
denote the phase term of $I$ in Equation~\ref{def:I} for a given
$\tilde{r}$ and $\omega$. The partial derivatives of $p(r,\theta;\tilde{r},\omega)$
(with $\tilde{r}$ and $\omega$ held fixed) are given by 
\begin{equation}
\frac{\partial p}{\partial r}=1-\tilde{r}\cos(\theta-\omega),\hspace{10pt}\frac{\partial p}{\partial\theta}=r\tilde{r}\sin(\theta-\omega).
\end{equation}
 Since $\mathcal{D}^{\epsilon}$ is bounded away from the origin $(0,0)$,
$\nabla p$ is well-defined and bounded and equals zero
only when $\tilde{r}=1$ and $\theta=\omega$. Since $\tilde{r}\not=1$
by assumption, no stationary point exists ($\nabla p\not=0$)
and hence we can expect $I(\tilde{r},\omega)\rightarrow0$ as $\tau\rightarrow0$
\cite{Cooke82,Jones58,Wong81}. Below, we show this result more explicitly.

Define a vector field $\mathbf{u}(r,\theta;\tilde{r},\omega)=\frac{\nabla p}{\|\nabla p\|^{2}}r$
at a fixed value of $\tilde{r}$ and $\omega$. Note that 
\begin{eqnarray}
\nabla\cdot\left[\mathbf{u}(r,\theta;\tilde{r},\omega)\exp\left(\frac{ip(r,\theta;\tilde{r},\omega)}{\tau}\right)\right]
& = &
\left(\nabla\cdot\mathbf{u}(r,\theta;\tilde{r},\omega)\right)\exp\left(\frac{ip(r,\theta;\tilde{r},\omega)}{\tau}\right)\nonumber
\\ 
 & + & \frac{i}{\tau}\exp\left(\frac{ip(r,\theta;\tilde{r},\omega)}{\tau}\right)r\label{eq:urelation}
\end{eqnarray}
 where the gradient operator $\nabla=\left(\frac{\partial}{\partial r},\frac{1}{r}\frac{\partial}{\partial\theta}\right)$.
Inserting Equation~\ref{eq:urelation} in Equation~\ref{def:I},
we get 
\begin{equation}
I(\tilde{r},\omega)=I^{(1)}(\tilde{r},\omega)-I^{(2)}(\tilde{r},\omega),\label{eq:IequalI1minusI2}
\end{equation}
 where 
\begin{eqnarray}
I^{(1)}(\tilde{r},\omega) & = & \frac{1}{2\pi il^{\epsilon}}\iint\limits
_{\mathcal{D}^{\epsilon}}\nabla\cdot\left(\mathbf{u}(r,\theta;\tilde{r},\omega)\exp\left(\frac{ip(r,\theta;\tilde{r},\omega)}{\tau}\right)\right)drd\theta,\,
\mathrm{and}\nonumber
\\ 
I^{(2)}(\tilde{r},\omega) & = & \frac{1}{2\pi il^{\epsilon}}\iint\limits
_{\mathcal{D}^{\epsilon}}\left(\nabla\cdot\mathbf{u}(r,\theta;\tilde{r},\omega)\right)
\exp\left(\frac{ip(r,\theta;\tilde{r},\omega)}{\tau}\right)drd\theta.
\end{eqnarray}
 Consider the integral $I^{(1)}(\tilde{r},\omega)$. From the divergence
theorem, we have 
\begin{equation}
I^{(1)}(\tilde{r},\omega)=\frac{1}{2\pi il^{\epsilon}}\int_{\Gamma}(\mathbf{u}^{T}\mathbf{n})\exp\left(\frac{ip(r,\theta;\tilde{r},\omega)}{\tau}\right)ds
\end{equation}
where $\Gamma$ is the positively oriented boundary of $\mathcal{D}^{\epsilon}$,
$s$ is the arc length of $\Gamma$ and $\mathbf{n}$ is the unit
outward normal of $\Gamma$. The boundary $\Gamma$ consists of two
disjoint regions, one along $r(\theta)=R^{(1)}(\theta)$ and another
along $r(\theta)=R^{(2)}(\theta)$. If the level curves of $p(r,\theta;\tilde{r},\omega)$
are tangential to $\Gamma$ only at a discrete set of locations giving
rise to stationary points of the second kind \cite{Wong89,Wong81,McClure91}---in other words, if $p(r,\theta;\tilde{r},\omega)$ is \emph{not constant}
along the boundary $\Gamma$ for any contiguous interval of $\theta$---then, using the one dimensional stationary phase approximation \cite{OlverBook74,OlverArticle74},
$I^{(1)}(\tilde{r},\omega)$ can be shown to be $O(\sqrt{\tau})$
and hence converges to zero as $\tau\rightarrow0$. Since the boundary
of $\mathcal{D}^{\epsilon}$ is composed of straight line segments
(specifically not arc-like), we can show that the level curves of
$p(r,\theta;\tilde{r},\omega)$ \emph{cannot} overlap with $\Gamma$
for a non-zero finite interval. (The next paragraph takes care of
this technical issue and can be skipped without loss of continuity.)

The level curves of $p(r,\theta;\tilde{r},\omega)$ are given by $R(\theta)(1-\tilde{r}\cos(\theta-\omega))=c$,
where $c$ is a constant. Recall that each of the two disjoint regions
of $\Gamma$ is composed of a finite sequence of line segments. For
the level curves of $p(r,\theta;\tilde{r},\omega)$ to coincide with
$\Gamma$ over a non-zero finite interval, $y(\theta)=R(\theta)\sin(\theta)=\frac{c\sin(\theta)}{1-\tilde{r}\cos(\theta-\omega)}$
and $x(\theta)=R(\theta)\cos(\theta)=\frac{c\cos(\theta)}{1-\tilde{r}\cos(\theta-\omega)}$
should satisfy the line equation $y=mx+b$ for some slope $m$ and
slope-intercept $b$, when $\theta$ varies over some contiguous interval
$\theta\in[\theta_{1},\theta_{2}]$. Plugging in the value of $y(\theta)$
and $x(\theta)$ into the line equation and expanding $\cos(\theta-\omega)$,
we have 
\begin{equation}
c\sin\left(\theta\right)=mc\cos(\theta)+b-b\tilde{r}[\cos(\theta)\cos(\omega)+\sin(\theta)\sin(\omega)].
\end{equation}
 Combining the terms, we get 
\begin{equation}
\sin(\theta)[c+b\tilde{r}\sin(\omega)]-\cos(\theta)[mc-b\tilde{r}\cos(\omega)]=b.
\end{equation}
 By defining $\lambda_{1}\equiv c+b\tilde{r}\sin(\omega)$ and
 $\lambda_{2}\equiv -(mc-b\tilde{r}\cos(\omega))$,
we see that $\sin(\theta)$ and $\cos(\theta)$ need to satisfy the
linear relation 
\begin{equation}
\lambda_{1}\sin(\theta)+\lambda_{2}\cos(\theta)=b\label{eq:sincoslinearrelation}
\end{equation}
 for $\theta\in[\theta_{1},\theta_{2}]$ in order for the level curves
of $p(r,\theta;\tilde{r},\omega)$ to overlap with the piece-wise
linear boundary $\Gamma$. As Equation~\ref{eq:sincoslinearrelation}
cannot be true for a finite interval of $\theta$, $I^{(1)}(\tilde{r},\omega)=O(\sqrt{\tau})$
as $\tau\rightarrow0$ and hence converges to zero in the limit.

Now $I^{(2)}(\tilde{r},\omega)$ has a similar form as the original
$I(\tilde{r},\omega)$ in Equation~\ref{def:I} with $r$ replaced
by $g_{1}(r,\theta;\tilde{r},\omega)=\left(\nabla\cdot\mathbf{u}\right)$.
Letting $\mathbf{u}_{1}(r,\theta;\tilde{r},\omega)=\frac{\nabla p}{\|\nabla p\|^{2}}g_{1}(r,\theta;\tilde{r},\omega)$,
from Equation~\ref{eq:urelation} and the divergence theorem, we get
\begin{eqnarray}
I^{(2)}(\tilde{r},\omega) & = & \frac{-\tau}{2\pi l^{\epsilon}}\int_{\Gamma}(\mathbf{u}_{1}^{T}\mathbf{n})\exp\left(\frac{ip(r,\theta;\tilde{r},\omega)}{\tau}\right)ds\nonumber \\
 & + & \frac{\tau}{2\pi l^{\epsilon}}\iint\limits _{\mathcal{D}^{\epsilon}}\left(\nabla\cdot\mathbf{u}_{1}(r,\theta;\tilde{r},\omega)\right)\exp\left(\frac{ip(r,\theta;\tilde{r},\omega)}{\tau}\right)drd\theta.
\end{eqnarray}
 As $I^{(2)}(\tilde{r},\omega)=O(\tau)$, it converges to zero as
$\tau\rightarrow0$. Applying the obtained results to Equation~\ref{eq:IequalI1minusI2},
we see that $I(\tilde{r},\omega)$ (and also $I_{k}(\tilde{r},\omega)$
defined in Equation~\ref{def:Ik}) $\rightarrow0$ as $\tau\rightarrow0$
which completes the proof. 

\end{proof}

Since Theorem~\ref{CircleTheorem} is true for any $0<\epsilon<\frac{1}{2}$,
it also holds as $\epsilon\rightarrow0$. As a corollary, we have
the following result:
\begin{corollary}
If $\tilde{r}\not=1$, then 
\begin{equation}
\lim_{\epsilon\rightarrow0}\lim_{\tau\rightarrow0}F_{\tau}^{\epsilon}(\tilde{r},\omega)=0.
\end{equation}

\end{corollary}

\section{Spatial Frequencies as Gradient Histogram Bins}

We now show that the squared magnitude of the Fourier transform of
the CWR ($\phi$) when polled close to the unit circle ($\tilde{r}=1$)
is approximately equal to the density function of the distance transform
gradients ($P$) with the approximation becoming increasingly tight
as $\tau\rightarrow0$.

The squared magnitude of the Fourier transform---also called its \emph{power
spectrum} \cite{Bracewell99}---is given by 
\begin{equation}
P_{\tau}^{\epsilon}(\tilde{r},\omega)\equiv|F_{\tau}^{\epsilon}(\tilde{r},\omega)|^{2}=F_{\tau}^{\epsilon}(\tilde{r},\omega)\overline{F_{\tau}^{\epsilon}(\tilde{r},\omega)}.\label{def:Ph}
\end{equation}
 By definition, $P_{\tau}^{\epsilon}(\tilde{r},\omega)\geq0$. From
Lemma~\ref{IntegralLemma}, we have 
\begin{equation}
\int_{0}^{2\pi}\int_{0}^{\infty}P_{\tau}^{\epsilon}(\tilde{r},\omega)\tilde{r}d\tilde{r}d\omega=1
\end{equation}
 \emph{independent} of $\tau$. Hence, $P_{\tau}^{\epsilon}(\tilde{r},\omega)$
can be treated as a density function for all values of $\tau$. We
earlier observed that the gradient density function of the unit vector
distance transform gradients is one-dimensional and defined over the
space of orientations $\omega$. For $P_{\tau}^{\epsilon}(\tilde{r},\omega)$
to behave as an orientation density function, it needs to be integrated
along the radial direction $\tilde{r}$. Since Theorem~\ref{CircleTheorem}
states that the Fourier transform values are concentrated only on
the unit circle $\tilde{r}=1$ and converges to zero elsewhere as
$\tau\rightarrow0$, it should be sufficient if the integration for
$\tilde{r}$ is done over a region very close to $\tilde{r}=1$. The
following theorem---the principal result in this paper---confirms
our observation. 
\begin{theorem}
\label{IntegralLimitTheorem} For any given $0<\epsilon<\frac{1}{2}$,
$0<\delta<1$, $\omega_{0}\in[0,2\pi)$ and $0<\Delta<2\pi$, 
\begin{equation}
\lim_{\tau\rightarrow0}\int_{\omega_{0}}^{\omega_{0}+\Delta}\int_{1-\delta}^{1+\delta}P_{\tau}^{\epsilon}(\tilde{r},\omega)\tilde{r}d\tilde{r}d\omega=\int_{\omega_{0}}^{\omega_{0}+\Delta}P(\omega)d\omega.
\end{equation}

\end{theorem}

\subsection{An Intuitive Examination of Theorem~\ref{IntegralLimitTheorem}}

Before we proceed with the formal proof, we again try and give an
intuitive explanation of why the theorem statement is true. The Fourier
transform of the CWR defined in Equation~\ref{eq:Fuvhbar} involves
two spatial integrals (over $x$ and $y$) which are converted into
polar coordinate integrals. The squared magnitude of the Fourier
transform (power spectrum), $P_{\tau}^{\epsilon}(\tilde{r},\omega)$,
involves multiplying the Fourier transform with its complex conjugate.
The complex conjugate is yet another $2D$ integral which we will
perform in polar coordinates. As the gradient density function
is one-dimensional and defined over the space of orientations, we
integrate the power spectrum along the radial direction close to the
unit circle $\tilde{r}=1$ (as $\delta\rightarrow0$). This is a fifth
integral. When we poll the power spectrum $P_{\tau}^{\epsilon}(\tilde{r},\omega)$
close to $\tilde{r}=1$, the two sinusoids, namely, $\exp\left(\frac{iS(x,y)}{\tau}\right)$
and $\exp\left(\frac{-i(ux+vy)}{\tau}\right)$ in Equation~\ref{eq:Fuvhbar}
are in resonance only when there is a \emph{perfect match} between
the orientation of each ray of the distance transform $S(x,y)$ and
the angle of the 2D spatial frequency ($\omega=\arctan\left(\frac{v}{u}\right)$).
All the grid locations $(x,y)$ having the same gradient orientation
\begin{equation}
\arctan\left(\frac{S_{y}}{S_{x}}\right)=\arctan\left(\frac{v}{u}\right)
\end{equation}
 \emph{cast} a vote only at their corresponding spatial frequency \textquotedbl{}histogram\textquotedbl{}
bin $\omega$. Since the histogram bin is generally populated by votes
from multiple grid locations, this leads to cross phase factors. Integrating
the power spectrum over a small range on the orientation (constituting
the sixth integral) helps in canceling out these phase factors giving
us the desired result when we take the limit as $\tau\rightarrow0$.
This integral and limit cannot be exchanged because the phase factors
will not otherwise cancel. The proof mainly deals with managing these
six integrals.

\subsection{Proof of Theorem~\ref{IntegralLimitTheorem}}

We now provide the formal proof of Theorem~\ref{IntegralLimitTheorem}.
For the sake of readability, we divide the proof into smaller subsections.
To achieve a good flow, we state major portions of our proof as lemmas
whose proofs are given in the appendix. We would like to emphasize
that these lemmas are meaningful only within the context of the 
proof and do not have much significance as stand-alone statements.
Important symbols used in the proof are adumbrated in Table~\ref{table:impsymbols}.

\begin{center}
\begin{table}[ht!]
\caption{Table of important symbols\label{table:impsymbols}}

\noindent \centering{}%
\begin{tabular}{|l|l|}
\hline 
Symbol  & Comments \tabularnewline
\hline\hline
$I$  & Integral of $P_{\tau}^{\epsilon}$ over the radial length $[1-\delta,1+\delta]$. \tabularnewline
\hline 
$g_{jk}$  & Integral over the variables $r$, $\theta$ and $\tilde{r}$ after
symmetry breaking. \tabularnewline
\hline 
$\gamma_{jk}$  & Phase term in the integral for $g_{jk}$. \tabularnewline
\hline 
$I_{jk}^{(1)},I_{jk}^{(2)}$  & Integrals for the main and the error terms of $g_{jk}$ respectively. \tabularnewline
 & $I$ is the sum of $I_{jk}^{(1)}$ and $I_{jk}^{(2)}$. \tabularnewline
\hline 
$p,q$  & Functions used in the definition of $I_{jk}^{(1)}$.\tabularnewline
 & $p$ represents the phase term.\tabularnewline
\hline 
$J_{jk}^{(1)},J_{jk}^{(2)},J_{jk}^{(3)}$  & Integrals obtained when $I_{jk}^{(1)}$ is split over the integral
range for $\theta^{\prime}$.\tabularnewline
\hline 
$\beta$  & Symbol used to divide the integral range for $\theta^{\prime}$ into
three integrals.\tabularnewline
 & The limit as $\beta\rightarrow0$ is considered in the proof.\tabularnewline
\hline 
$G_{jk}$  & Result of integrating over $\theta^{\prime}$ while evaluating $J_{jk}^{(1)}$.\tabularnewline
\hline 
$\psi_{jk},\chi$  & Integrals for the main and the error terms of $G_{jk}$ respectively. \tabularnewline
 & $J_{jk}^{(1)}$ is the sum of $\psi_{jk}$ and $\chi$.\tabularnewline
\hline 
$\epsilon_{3},\epsilon_{4}$  & Error terms used to define $\chi$.\tabularnewline
\hline 
\end{tabular}
\end{table}

\par\end{center}

\noindent First, observe that 
\begin{equation}
\overline{F_{\tau}^{\epsilon}(\tilde{r},\omega)}\equiv\sum_{k=1}^{K}\frac{\overline{C_{k}}}{2\pi\tau l^{\epsilon}}\int_{0}^{2\pi}\int_{R_{k}^{(1)}(\theta^{\prime})}^{R_{k}^{(2)}(\theta^{\prime})}\exp\left(-\frac{ir^{\prime}}{\tau}\left[1-\tilde{r}\cos(\theta^{\prime}-\omega)\right]\right)r^{\prime}dr^{\prime}d\theta^{\prime}.
\end{equation}
 Define 
\begin{equation}
I(\omega)\equiv\int_{1-\delta}^{1+\delta}P_{\tau}^{\epsilon}(\tilde{r},\omega)\tilde{r}d\tilde{r}=\int_{1-\delta}^{1+\delta}F_{\tau}^{\epsilon}(\tilde{r},\omega)\overline{F_{\tau}^{\epsilon}(\tilde{r},\omega)}\tilde{r}d\tilde{r}.\label{def:Idef1}
\end{equation}
 As $\tau\rightarrow0$, we show that $I(\omega)$ approaches the
density function of the gradients of $S(x,y)$. Note that the integral
in Equation~\ref{def:Idef1} is over the interval $[1-\delta,1+\delta]$,
where $\delta>0$ can be made arbitrarily small (as $\tau\rightarrow0$)
and this is due to Theorem~\ref{CircleTheorem}.

Recall that in order to evaluate $I(\omega)$, we need to perform five
integrals, four to obtain the power spectrum $P_{\tau}^{\epsilon}(\tilde{r},\omega)$
and a fifth along the radial direction $\tilde{r}$ over $[1-\delta,1+\delta]$
which is close to the unit circle $\tilde{r}=1$. An easy way to
compute $I(\omega)$ in the limit $\tau\rightarrow0$ would be to
apply a 5D stationary phase approximation \cite{Wong89}. Unfortunately,
the 5D stationary phase approximation \emph{cannot} be directly employed
in our case for reasons detailed in Appendix~\ref{sec:pitfalls5Dstationaryphase}.

\subsection*{Breaking the symmetry of the integral}
As described in Section~\ref{sec:symmetrybreaking}, we propose to solve for $I(\omega)$ in Equation~\ref{def:Idef1}
by breaking the symmetry of the integral. We fix the conjugate variables
$r^{\prime}$ and $\theta^{\prime}$ and perform the integration only
with respect to the other three variables namely $r$, $\theta$ and
$\tilde{r}$. To this end, let 
\begin{equation}
I(\omega)=\sum_{j=1}^{K}\sum_{k=1}^{K}\frac{1}{(2\pi\tau l^{\epsilon})^{2}}\int_{0}^{2\pi}\int_{R_{k}^{(1)}(\theta^{\prime})}^{R_{k}^{(2)}(\theta^{\prime})}\exp\left(\frac{-ir^{\prime}}{\tau}\right)g_{jk}(r^{\prime},\theta^{\prime};\omega)r^{\prime}dr^{\prime}d\theta^{\prime},\label{eq:Idef2}
\end{equation}
 where 
\begin{eqnarray}
g_{jk}(r^{\prime},\theta^{\prime};\omega) & = & \int_{1-\delta}^{1+\delta}\int_{0}^{2\pi}\int_{R_{j}^{(1)}(\theta)}^{R_{j}^{(2)}(\theta)}\exp\left\{ \frac{i}{\tau}\gamma_{jk}(r,\theta,\tilde{r};r^{\prime},\theta^{\prime},\omega)\right\} f_{2}(r,\tilde{r})drd\theta d\tilde{r}.\label{def:g}
\end{eqnarray}
Here,
\begin{eqnarray}
\gamma_{jk}(r,\theta,\tilde{r};r^{\prime},\theta^{\prime},\omega) & = & r\left[1-\tilde{r}\cos(\theta-\omega)\right]+r^{\prime}\tilde{r}\cos(\theta^{\prime}-\omega)\nonumber \\
 &  & -\tilde{r}\left[\cos(\omega)(x_{j}-x_{k})+\sin(\omega)(y_{j}-y_{k})\right]\label{def:gamma}
\end{eqnarray}
 and 
\begin{equation}
f_{2}(r,\tilde{r})=r\tilde{r}.
\end{equation}
In the definition of $\gamma_{jk}(r,\theta,\tilde{r};r^{\prime},\theta^{\prime},\omega)$
in Equation~\ref{def:gamma}, $\omega,r^{\prime}$ and $\theta^{\prime}$
are held fixed. Similarly, in the definition of $g_{jk}(r^{\prime},\theta^{\prime};\omega)$
in Equation~\ref{def:g}, $\omega$ is a constant. The phase term
of the quantity $C_{j}\overline{C_{k}}$ (Equation~\ref{eq:phasetermCjCk})
is absorbed in $\gamma_{jk}$ and pursuant to Fubini's theorem \cite{Fubini58},
the integration with respect to $\tilde{r}$ can be considered before
the integration over $r^{\prime}$ and $\theta^{\prime}$. Define 
\begin{equation}
r_{jk}(r^{\prime},\theta^{\prime};\omega)\equiv r^{\prime}\cos(\theta^{\prime}-\omega)-[\cos(\omega)(x_{j}-x_{k})+\sin(\omega)(y_{j}-y_{k})].
\end{equation}
This leads to the following lemma.
\begin{lemma}
\label{lemma:gjk} If $r_{jk}(r^{\prime},\theta^{\prime};\omega)>0$,
then as $\tau\rightarrow0$, 
\begin{eqnarray}
g_{jk}(r^{\prime},\theta^{\prime};\omega) & = & (2\pi\tau)^{\frac{3}{2}}\sqrt{r_{jk}(r^{\prime},\theta^{\prime};\omega)}\exp\left(\frac{ir_{jk}(r^{\prime},\theta^{\prime};\omega)}{\tau}+\frac{i\pi}{4}\right)\nonumber \\
 &  & +\tau^{\kappa}\xi_{jk}(r^{\prime},\theta^{\prime};\omega)\label{eq:gjk1}
\end{eqnarray}
where $\kappa\geq2$ and $\xi_{jk}(r^{\prime},\theta^{\prime};\omega)$
is some bounded continuous function which includes the contributions
from the boundary. If $r_{jk}(r^{\prime},\theta^{\prime};\omega)\leq0$,
then as $\tau\rightarrow0$, $g_{jk}(r^{\prime},\theta^{\prime};\omega)=0$.
\\

\end{lemma}
The proof of Lemma~\ref{lemma:gjk}---obtained using a three
dimensional stationary phase approximation---is available in Appendix~\ref{sec:prooflemmagjk}.
Note that for $j=k$ and $\theta^{\prime}$ close to $\omega$, $r_{jk}(r^{\prime},\theta^{\prime};\omega)>0$
and hence $g_{jk}(r^{\prime},\theta^{\prime};\omega)\not=0$. Below,
we show that the only pertinent scenarios that need consideration
are $\theta^{\prime}$ close to $\omega$ and $j=k$. When $\theta^{\prime}$
is away from $\omega$ or $j\not=k$, the integral $g_{jk}(r^{\prime},\theta^{\prime};\omega)$
vanishes. Hence, for the sake of readability of our proof, we let $g_{jk}(r^{\prime},\theta^{\prime};\omega)$
take the most general form given in Equation~\ref{eq:gjk1} for all
values of $r^{\prime}$ and $\theta^{\prime}$.

\subsection*{Determining $I(\omega)$}

Substituting the value of $g_{jk}(r^{\prime},\theta^{\prime};\omega)$
into Equation~\ref{eq:Idef2}, as $\tau\rightarrow0$, we get 
\begin{equation}
I(\omega)=\sum_{j=1}^{K}\sum_{k=1}^{K}\left\{ \frac{\eta_{jk}(\omega)}{L^{\epsilon}}I_{jk}^{(1)}(\omega)+I_{jk}^{(2)}(\omega)\right\} \label{eq:Idef3}
\end{equation}
 where 
\begin{eqnarray}
I_{jk}^{(1)}(\omega) & = & \frac{1}{\sqrt{2\pi\tau}}\int_{0}^{2\pi}\int_{R_{k}^{(1)}(\theta^{\prime})}^{R_{k}^{(2)}(\theta^{\prime})}\exp\left(\frac{ip(r^{\prime},\theta^{\prime};\omega)}{\tau}\right)q(r^{\prime},\theta^{\prime};\omega)dr^{\prime}d\theta^{\prime},\nonumber \\
I_{jk}^{(2)}(\omega) & = & \int_{0}^{2\pi}\int_{R_{k}^{(1)}(\theta^{\prime})}^{R_{k}^{(2)}(\theta^{\prime})}\exp\left(\frac{-ir^{\prime}}{\tau}\right)r^{\prime}\frac{1}{(2\pi l^{\epsilon})^{2}}\tau^{\kappa-2}\xi_{jk}(r^{\prime},\theta^{\prime};\omega)dr^{\prime}d\theta^{\prime},\nonumber \\
\eta_{jk}(\omega) & = & \exp\left(\frac{-i\alpha_{jk}(\omega)}{\tau}+\frac{i\pi}{4}\right),\nonumber \\
\alpha_{jk}(\omega) & = & \cos(\omega)(x_{j}-x_{k})+\sin(\omega)(y_{j}-y_{k}),\nonumber \\
p(r^{\prime},\theta^{\prime};\omega) & = & -r^{\prime}[1-\cos(\theta^{\prime}-\omega)]\mbox{ and}\nonumber \\
q(r^{\prime},\theta^{\prime};\omega) & = & r^{\prime}\sqrt{r^{\prime}\cos(\theta^{\prime}-\omega)-\alpha_{jk}(\omega)}.\label{eqn:symdefns}
\end{eqnarray}
 In the definition of the functions $p(r^{\prime},\theta^{\prime};\omega)$
and $q(r^{\prime},\theta^{\prime};\omega)$, $\omega$ is held fixed.
Since $\kappa\geq2$, by the Riemann-Lebesgue lemma, we have $\lim_{\tau\rightarrow0}I_{jk}^{(2)}=0$
and from the Lebesgue dominated convergence theorem, it follows that
\begin{equation}
\lim_{\tau\rightarrow0}\int_{\omega_{0}}^{\omega_{0}+\Delta}I_{jk}^{(2)}(\omega)=\int_{0}^{2\pi}\lim_{\tau\rightarrow0}I_{jk}^{(2)}(\omega)=0.
\end{equation}
 Using the above result in Equation~\ref{eq:Idef3}, we get 
\begin{equation}
\lim_{\tau\rightarrow0}\int_{\omega_{0}}^{\omega_{0}+\Delta}I(\omega)d\omega=\sum_{j=1}^{K}\sum_{k=1}^{K}\lim_{\tau\rightarrow0}\int_{\omega_{0}}^{\omega_{0}+\Delta}\frac{\eta_{jk}(\omega)}{L^{\epsilon}}I_{jk}^{(1)}(\omega)d\omega.\label{eq:Idef4}
\end{equation}

\subsection*{Splitting the integral over $\theta^{\prime}$ into three disconnected
regions}

Consider the integral $I_{jk}^{(1)}(\omega)$. As essential contributions
to it come only from the stationary points of $p(r^{\prime},\theta^{\prime};\omega)$
\cite{Jones58,Wong81,Cooke82} (with $\omega$ held fixed), we first
determine its critical (stationary) point(s). The partial derivatives of $p(r^{\prime},\theta^{\prime};\omega)$
at a fixed value of $\omega$ are given by 
\begin{equation}
\frac{\partial p}{\partial r^{\prime}}=-1+\cos(\theta^{\prime}-\omega),\hspace{10pt}\frac{\partial p}{\partial\theta^{\prime}}=-r^{\prime}\sin(\theta^{\prime}-\omega).
\end{equation}
 For $\nabla p=0$, we must have $\theta^{\prime}=\omega$.
Hence, in order to evaluate $I_{jk}^{(1)}(\omega)$, we find it useful
to divide the integral range $[0,2\pi)$ for $\theta^{\prime}$ into
three disjoint regions namely $[0,\omega-\beta)$, $[\omega-\beta,\omega+\beta]$
and $(\omega+\beta,2\pi)$ for a fixed $\beta>0$, and write 
\begin{equation}
I_{jk}^{(1)}(\omega)=J_{jk}^{(1)}(\beta,\omega)+J_{jk}^{(2)}(\beta,\omega)+J_{jk}^{(3)}(\beta,\omega)\label{eq:I1jk}
\end{equation}
 where 
\begin{eqnarray}
J_{jk}^{(1)}(\beta,\omega) & = & \frac{1}{\sqrt{2\pi\tau}}\int_{\omega-\beta}^{\omega+\beta}\int_{R_{k}^{(1)}(\theta^{\prime})}^{R_{k}^{(2)}(\theta^{\prime})}\exp\left(\frac{ip(r^{\prime},\theta^{\prime};\omega)}{\tau}\right)q(r^{\prime},\theta^{\prime};\omega)dr^{\prime}d\theta^{\prime},\nonumber \\
J_{jk}^{(2)}(\beta,\omega) & = & \frac{1}{\sqrt{2\pi\tau}}\int_{0}^{\omega-\beta}\int_{R_{k}^{(1)}(\theta^{\prime})}^{R_{k}^{(2)}(\theta^{\prime})}\exp\left(\frac{ip(r^{\prime},\theta^{\prime};\omega)}{\tau}\right)q(r^{\prime},\theta^{\prime};\omega)dr^{\prime}d\theta^{\prime},\,\mathrm{and}\nonumber \\
J_{jk}^{(3)}(\beta,\omega) & = & \frac{1}{\sqrt{2\pi\tau}}\int_{\omega+\beta}^{2\pi}\int_{R_{k}^{(1)}(\theta^{\prime})}^{R_{k}^{(2)}(\theta^{\prime})}\exp\left(\frac{ip(r^{\prime},\theta^{\prime};\omega)}{\tau}\right)q(r^{\prime},\theta^{\prime};\omega)dr^{\prime}d\theta^{\prime}.
\end{eqnarray}
Since the above relation is true for \emph{any} $\beta>0$, we can
let $\beta\rightarrow0$ (after we take the limit $\tau\rightarrow0$).
Fix a $\beta$ close to zero and consider the above integrals as $\tau\rightarrow0$.
Then we obtain: 
\begin{lemma}
\label{lemma:Ijk1equalsJjk1} 
\begin{equation}
\lim_{\tau\rightarrow0}\int_{\omega_{0}}^{\omega_{0}+\Delta}\frac{\eta_{jk}(\omega)}{L^{\epsilon}}I_{jk}^{(1)}(\omega)d\omega=\lim_{\beta\rightarrow0}\lim_{\tau\rightarrow0}\int_{\omega_{0}}^{\omega_{0}+\Delta}\frac{\eta_{jk}(\omega)}{L^{\epsilon}}J_{jk}^{(1)}(\beta,\omega)d\omega.
\end{equation}
 
\end{lemma}
\noindent The proof is available in Appendix~\ref{sec:Ijk1equalsJjk1}.

\subsection*{Interchanging the order of integration between $r^{\prime}$ and
$\theta^{\prime}$}

We now evaluate $J_{jk}^{(1)}(\beta,\omega)$ by interchanging the
order of integration between $r^{\prime}$ and $\theta^{\prime}$
which requires us to rewrite $\theta^{\prime}$ as a function of $r^{\prime}$.
Recall that the boundaries of each $\mathcal{D}_{k}^{\epsilon}$ along
$r(\theta^{\prime})=R_{k}^{(1)}(\theta^{\prime})$ and $r(\theta^{\prime})=R_{k}^{(2)}(\theta^{\prime})$
respectively are composed of a finite sequence of straight line segments.
In order to evaluate $J_{jk}^{(1)}(\beta,\omega)$, we need to consider
these boundaries only within the precincts of the angles $[\omega-\beta,\omega+\beta]$
at each $\mathcal{D}_{k}^{\epsilon}$. But for sufficiently small
$\beta$, we observe that for \emph{every} $\omega\in[0,2\pi)$, when
we consider these boundaries (along $R_{k}^{(1)}(\theta^{\prime})$
and $R_{k}^{(2)}(\theta^{\prime})$ respectively) within the angles
$[\omega-\beta,\omega+\beta]$, they are composed of \emph{at most}
two line segments as portrayed in Figure~\ref{fig:twoangle}.

\begin{figure}[ht!]
\begin{centering}
\includegraphics[width=0.4\textwidth]{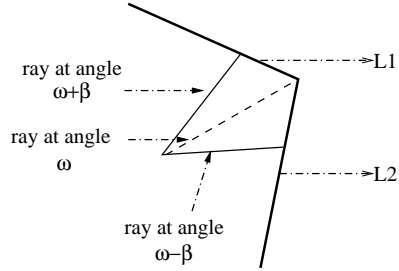} 
\par\end{centering}

\caption{Boundary considered within the angles $[\omega-\beta,\omega+\beta]$
is comprised of at most two line segments $L_{1}$ and $L_{2}$.}

\label{fig:twoangle} 
\end{figure}

Over each line segment, $r^{\prime}(\theta^{\prime})$ is either strictly
monotonic (strictly increases or strictly decreases) or has exactly
one critical point (strictly decreases, attains a minimum and then
strictly increases) as described in Figure~\ref{fig:rvstheta}.

\begin{figure}[ht!]
\begin{centering}
\includegraphics[width=0.8\textwidth]{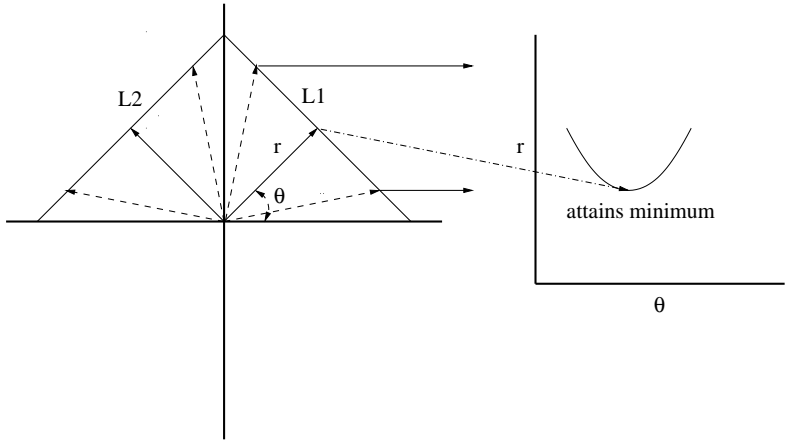} 
\par\end{centering}

\caption{Plot of radial length ($r$) vs angle ($\theta$).}

\label{fig:rvstheta} 
\end{figure}

Hence, it follows that for sufficiently small $\beta$, $\theta^{\prime}$
rewritten as a function of $r^{\prime}$ is composed of at most three
disconnected regions (as seen in Figure~\ref{fig:disconRegions}).

\begin{figure}[ht!]
\begin{centering}
\includegraphics[width=0.4\textwidth]{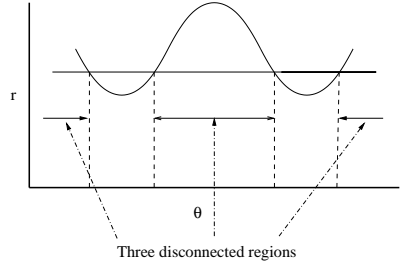} 
\par\end{centering}

\caption{Three disconnected regions for the angle ($\theta$).}

\label{fig:disconRegions} 
\end{figure}

Let $\mathcal{B}(r^{\prime})\subseteq[\omega-\beta,\omega+\beta]$
denote the integration region for $\theta^{\prime}(r^{\prime})$.
Treating $\theta^{\prime}$ as a function of $r^{\prime}$, the integral
$J_{jk}^{(1)}(\beta,\omega)$ can be rewritten as 
\begin{equation}
J_{jk}^{(1)}(\beta,\omega)=\int_{r_{k}^{(1)}(\beta,\omega)}^{r_{k}^{(2)}(\beta,\omega)}G_{jk}(r^{\prime},\omega)dr^{\prime},\label{eq:J1jk}
\end{equation}
 where 
\begin{eqnarray}
r_{k}^{(1)}(\beta,\omega) & = & \inf\{R_{k}^{(1)}(\theta^{\prime})\},\,\mathrm{and}\nonumber \\
r_{k}^{(2)}(\beta,\omega) & = & \sup\{R_{k}^{(2)}(\theta^{\prime})\}\label{eq:rk1andrk2}
\end{eqnarray}
 with $\theta^{\prime}\in[\omega-\beta,\omega+\beta]$ and 
\begin{equation}
G_{jk}(r^{\prime},\omega)=\frac{1}{\sqrt{2\pi\tau}}\int_{\mathcal{B}(r^{\prime})}\exp\left(\frac{ip(r^{\prime},\theta^{\prime},\omega)}{\tau}\right)q(r^{\prime},\theta^{\prime},\omega)d\theta^{\prime}.
\end{equation}
 Note that while evaluating the integral $G_{jk}(r^{\prime},\omega)$,
$r^{\prime}$ and $\omega$ are held fixed. As contributions to $G_{jk}(r^{\prime},\omega)$
come only from the stationary points of $p(r^{\prime},\theta^{\prime},\omega)$
(with $r^{\prime}$ and $\omega$ held fixed) as $\tau\rightarrow0$,
we evaluate $\frac{\partial p}{\partial\theta^{\prime}}=-r^{\prime}\sin(\theta^{\prime}-\omega)$
and for it to vanish, we require $\theta^{\prime}=\omega$. Moreover
\begin{eqnarray}
\left.\frac{\partial^{2}p}{\partial\theta^{\prime2}}\right\vert_{\omega} & = & -r^{\prime},\nonumber \\
p(r^{\prime},\omega,\omega) & = & 0,\mbox{ and}\nonumber \\
q(r^{\prime},\omega,\omega) & = & r^{\prime}\sqrt{r^{\prime}-\alpha_{jk}(\omega)}.
\end{eqnarray}
For the given $r^{\prime}$, if $\omega\notin\mathcal{B}(r^{\prime})$,
no stationary points exist. Using integration by parts, $G_{jk}(r^{\prime},\omega)$
can be shown to be $\epsilon_{3}(r^{\prime},\omega,\tau)=O(\sqrt{\tau})$,
which can be uniformly bounded by a function of $r^{\prime}$ and
$\omega$ for small values of $\tau$.

If $\omega\in\mathcal{B}(r^{\prime})$, then using the one dimensional
stationary phase approximation \cite{OlverBook74,OlverArticle74},
it can be shown that 
\begin{equation}
G_{jk}(r^{\prime},\omega)=\exp\left(\frac{-i\pi}{4}\right)\sqrt{r^{\prime}}\sqrt{r^{\prime}-\alpha_{jk}(\omega)}+\epsilon_{4}(r^{\prime},\omega,\tau),
\end{equation}
 where $\epsilon_{4}(r^{\prime},\omega,\tau)$ can be uniformly bounded
by a function of $r^{\prime}$ and $\omega$ for small values of $\tau$
and converges to zero as $\tau\rightarrow0$. Here, we assume that
the stationary point $\theta^{\prime}=\omega$ lies in the \emph{interior}
of $\mathcal{B}(r^{\prime})$ and not on the boundary as there can
be at most finite (actually 2) values of $r^{\prime}$ (with Lebesgue
measure zero) for which $\theta^{\prime}=\omega$ can lie on the boundary
of $\mathcal{B}(r^{\prime})$.

\subsection*{Computing the integral over $\omega$ and $r^{\prime}$}

Let $r_{k}^{(-)}(\beta,\omega)\geq r_{k}^{(1)}(\beta,\omega)$ and
$r_{k}^{(+)}(\beta,\omega)\leq r_{k}^{(2)}(\beta,\omega)$ be the
values of $r^{\prime}$ such that when $r_{k}^{(-)}(\beta,\omega)<r^{\prime}<r_{k}^{(+)}(\beta,\omega)$,
the stationary point $\theta^{\prime}=\omega$ lies in the interior
of $\mathcal{B}(r^{\prime})$. Substituting the value of $G_{jk}(r^{\prime},\omega)$
into Equation~\ref{eq:J1jk} and using the definitions of $\eta_{jk}(\omega)$
and $\alpha_{jk}(\omega)$ from Equation~\ref{eqn:symdefns}, we
get 
\begin{equation}
\int_{\omega_{0}}^{\omega_{0}+\Delta}\frac{\eta_{jk}(\omega)}{L^{\epsilon}}J_{jk}^{(1)}(\beta,\omega)d\omega=\psi_{jk}(\beta)+\int_{\omega_{0}}^{\omega_{0}+\Delta}\frac{\eta_{jk}(\omega)}{L^{\epsilon}}\left\{ \int_{r_{k}^{(1)}(\beta,\omega)}^{r_{k}^{(2)}(\beta,\omega)}\chi(r^{\prime},\omega,\tau)\hspace{3pt}dr^{\prime}\right\} d\omega,\label{eq:J1jk_2}
\end{equation}
 where 
\begin{equation}
\psi_{jk}(\beta)=\frac{1}{{L^{\epsilon}}}\int_{\omega_{0}}^{\omega_{0}+\Delta}\exp\left(\frac{-i\alpha_{jk}(\omega)}{\tau}\right)\int_{r_{k}^{(-)}(\beta,\omega)}^{r_{k}^{(+)}(\beta,\omega)}\sqrt{r^{\prime}}\sqrt{r^{\prime}-\alpha_{jk}(\omega)}dr^{\prime}d\omega
\end{equation}
 and 
\[
\chi(r^{\prime},\omega,\tau)=\left\{ \begin{array}{cc}
\epsilon_{4}(r^{\prime},\omega,\tau), & r_{k}^{(-)}(\beta,\omega)<r^{\prime}<r_{k}^{(+)}(\beta,\omega),\\
\epsilon_{3}(r^{\prime},\omega,\tau), & r^{\prime}<r_{k}^{(-)}(\beta,\omega)\mbox{ or }r_{k}^{(+)}(\beta,\omega)<r^{\prime}.
\end{array}\right.
\]

Since $|\eta_{jk}(\omega)|=1$ and $\chi(r^{\prime},\omega,\tau)$
can be uniformly bounded by a function $r^{\prime}$ and $\omega$
for small values of $\tau$, by the Lebesgue dominated convergence
theorem we have 
\begin{eqnarray}
 &  &
  \lim_{\tau\rightarrow0}\int_{\omega_{0}}^{\omega_{0}+\Delta}\frac{\eta_{jk}(\omega)}{L^{\epsilon}}\left\{
  \int_{r_{k}^{(1)}(\beta,\omega)}^{r_{k}^{(2)}(\beta,\omega)}\chi(r^{\prime},\omega,\tau)\hspace{3pt}dr^{\prime}\right\}
  d\omega\nonumber \\ 
 &  &
  =\int_{\omega_{0}}^{\omega_{0}+\Delta}\frac{\eta_{jk}(\omega)}{L^{\epsilon}}\left\{
  \int_{r_{k}^{(1)}(\beta,\omega)}^{r_{k}^{(2)}(\beta,\omega)}\lim_{\tau\rightarrow0}\chi(r^{\prime},\omega,\tau)\hspace{3pt}dr^{\prime}\right\}
  d\omega=0. 
\end{eqnarray}
 This leaves us having to prove the following result:
\begin{lemma}
\label{lemma:psijkequalsP} 
\begin{equation}
\sum_{j=1}^{K}\sum_{k=1}^{K}\lim_{\beta\rightarrow0}\lim_{\tau\rightarrow0}\psi_{jk}(\beta)=\int_{\omega_{0}}^{\omega_{0}+\Delta}
P(\omega)d\omega.  
\end{equation}
 
\end{lemma}
\noindent The proof of this lemma is given in Appendix~\ref{sec:psijkequalsP}.
This completes the proof of Theorem~\ref{IntegralLimitTheorem}.

We would like to give a short recap of our proof. Beginning with the
definition of $I(\omega)$ in Equation~\ref{def:Idef1}, Lemma~\ref{lemma:gjk}
and the statements following it lead to the relation~\ref{eq:Idef4},
namely, 
\begin{equation}
\lim_{\tau\rightarrow0}\int_{\omega_{0}}^{\omega_{0}+\Delta}I(\omega)d\omega=\sum_{j=1}^{K}\sum_{k=1}^{K}\lim_{\tau\rightarrow0}\int_{\omega_{0}}^{\omega_{0}+\Delta}\frac{\eta_{jk}(\omega)}{L^{\epsilon}}I_{jk}^{(1)}(\omega)d\omega.
\end{equation}
 From Lemma~\ref{lemma:Ijk1equalsJjk1}, it follows that 
\begin{eqnarray}
 &  &
  \sum_{j=1}^{K}\sum_{k=1}^{K}\lim_{\tau\rightarrow0}\int_{\omega_{0}}^{\omega_{0}+\Delta}\frac{\eta_{jk}(\omega)}{L^{\epsilon}}
  I_{jk}^{(1)}(\omega)d\omega=\nonumber
  \\ 
 &  &
  \sum_{j=1}^{K}\sum_{k=1}^{K}\lim_{\beta\rightarrow0}\lim_{\tau\rightarrow0}\int_{\omega_{0}}^{\omega_{0}+\Delta}\frac{\eta_{jk}
    (\omega)}{L^{\epsilon}}J_{jk}^{(1)}(\beta,\omega)d\omega. 
\end{eqnarray}
 Interchanging the order of integration between $r^{\prime}$ and
$\theta^{\prime}$, we showed that 
\begin{equation}
\sum_{j=1}^{K}\sum_{k=1}^{K}\lim_{\beta\rightarrow0}\lim_{\tau\rightarrow0}\int_{\omega_{0}}^{\omega_{0}+\Delta}\frac{\eta_{jk}(\omega)}{L^{\epsilon}}J_{jk}^{(1)}(\beta,\omega)d\omega=\sum_{j=1}^{K}\sum_{k=1}^{K}\lim_{\beta\rightarrow0}\lim_{\tau\rightarrow0}\psi_{jk}(\beta).
\end{equation}
 Finally, the application of Lemma~\ref{lemma:psijkequalsP} gives
the desired result of Theorem~\ref{IntegralLimitTheorem}, namely,
\begin{eqnarray}
\lim_{\tau\rightarrow0}\int_{\omega_{0}}^{\omega_{0}+\Delta}\int_{1-\delta}^{1+\delta}P_{\tau}^{\epsilon}(\tilde{r},\omega)\tilde{r}d\tilde{r}d\omega & = & \lim_{\tau\rightarrow0}\int_{\omega_{0}}^{\omega_{0}+\Delta}I(\omega)d\omega\nonumber \\
 & = & \sum_{j=1}^{K}\sum_{k=1}^{K}\lim_{\beta\rightarrow0}\lim_{\tau\rightarrow0}\psi_{jk}(\beta)\nonumber \\
 & = & \int_{\omega_{0}}^{\omega_{0}+\Delta}P(\omega)d\omega.
\end{eqnarray}

\section{Results Stemming from the Main Theorem}

As an implication of Theorem~\ref{IntegralLimitTheorem}, we have
the following corollary. 
\begin{corollary}
\label{DensityCorollary} For any given $0<\delta<1$, $\omega_{0}\in[0,2\pi)$,
\begin{equation}
\lim_{\epsilon\rightarrow0}\lim_{\Delta\rightarrow0}\frac{1}{\Delta}\lim_{\tau\rightarrow0}\int_{\omega_{0}}^{\omega_{0}+\Delta}
\left\{
\int_{1-\delta}^{1+\delta}P_{\tau}^{\epsilon}(\tilde{r},\omega)\tilde{r}\hspace{3pt}d\tilde{r}\right\}
d\omega=P(\omega_{0}). 
\end{equation}
 \end{corollary}
\begin{proof}
From Equation~\ref{def:densityfunc}, we have 
\begin{equation}
\lim_{\Delta\rightarrow0}\frac{1}{\Delta}\int_{\omega_{0}}^{\omega_{0}+\Delta}P(\omega)d\omega=\lim_{\Delta\rightarrow0}\frac{F(
  \omega_{0}\leq\omega\leq\omega_{0}+\Delta)}{\Delta}=P(\omega_{0}).
\end{equation}
 Since Theorem~\ref{IntegralLimitTheorem} is true for any $0<\epsilon<\frac{1}{2}$,
it also holds as $\epsilon\rightarrow0$. The result then follows
immediately. 

\end{proof}
Theorem~\ref{IntegralLimitTheorem} also entails the following lemma. 
\begin{lemma}
\label{IntegralLimitLemma} For any given $0<\epsilon<\frac{1}{2}$,
$0<\delta<1$, 
\begin{equation}
\lim_{\tau\rightarrow0}\int_{0}^{2\pi}\int_{1-\delta}^{1+\delta}P_{\tau}^{\epsilon}(\tilde{r},\omega)\tilde{r}d\tilde{r}d\omega=1.
\end{equation}
 \end{lemma}
\begin{proof}
Since the result shown in Theorem~\ref{IntegralLimitTheorem} holds
good for any $\omega_{0}$ and $\Delta$, we may choose $\omega_{0}=0$
and $\Delta=2\pi$. Using Equation~\ref{eq:PInt} the result follows
immediately as 
\begin{equation}
\lim_{\tau\rightarrow0}\int_{0}^{2\pi}\int_{1-\delta}^{1+\delta}P_{\tau}^{\epsilon}(\tilde{r},\omega)\tilde{r}d\tilde{r}d\omega=\int_{0}^{2\pi}P(\omega)d\omega=1.
\end{equation}
 
\end{proof}
\noindent Lemmas~\ref{IntegralLimitLemma} and \ref{IntegralLemma}
leads to the following corollaries. 
\begin{corollary}
\label{ZeroCorollaryEntireInt} For any given $0<\epsilon<\frac{1}{2}$,
$0<\delta<1$, 
\begin{equation}
\lim_{\tau\rightarrow0}\int_{0}^{2\pi}\left\{
\int_{0}^{1-\delta}P_{\tau}^{\epsilon}(\tilde{r},\omega)\hspace{3pt}\tilde{r}d\tilde{r}+\int_{1+\delta}^{\infty}
P_{\tau}^{\epsilon}(\tilde{r},\omega)\hspace{3pt}\tilde{r}d\tilde{r}\right\}d\omega=0. 
\end{equation}
 \end{corollary}
\begin{proof}
From Lemma~\ref{IntegralLemma}, we have for any $\tau>0$ and $0<\epsilon<\frac{1}{2}$,
\begin{equation}
\int_{0}^{2\pi}\int_{0}^{\infty}P_{\tau}^{\epsilon}(\tilde{r},\omega)\tilde{r}d\tilde{r}d\omega=1.
\end{equation}
 For the given $0<\delta<1$, dividing the integral range $(0,\infty)$
for $\tilde{r}$ into three disjoint regions namely $(0,1-\delta)$,
$[1-\delta,1+\delta]$ and $(1+\delta,\infty)$ and letting $\tau\rightarrow0$,
we have

\[
\lim_{\tau\rightarrow0}\int_{0}^{2\pi}\left\{ \int_{0}^{1-\delta}P_{\tau}^{\epsilon}(\tilde{r},\omega)\hspace{3pt}\tilde{r}d\tilde{r}+\int_{1-\delta}^{1+\delta}P_{\tau}^{\epsilon}(\tilde{r},\omega)\hspace{3pt}\tilde{r}d\tilde{r}+\int_{1+\delta}^{\infty}P_{\tau}^{\epsilon}(\tilde{r},\omega)\hspace{3pt}\tilde{r}d\tilde{r}\right\} d\omega=1.
\]

Pursuant to Lemma~\ref{IntegralLimitLemma}, the limit 
\begin{equation}
\lim_{\tau\rightarrow0}\int_{0}^{2\pi}\int_{1-\delta}^{1+\delta}P_{\tau}^{\epsilon}(\tilde{r},\omega)\hspace{3pt}\tilde{r}d\tilde{r}d\omega
\end{equation}
 \emph{exists} and equals 1. The result then follows.  

\end{proof}
\begin{corollary}
\label{ZeroCorollaryFiniteInt} For any given $0<\epsilon<\frac{1}{2}$,
$0<\delta<1$, $\omega_{0}\in[0,2\pi)$ and $0<\Delta<2\pi$, 
\begin{equation}
\lim_{\tau\rightarrow0}\int_{\omega_{0}}^{\omega_{0}+\Delta}\left\{ \int_{0}^{1-\delta}P_{\tau}^{\epsilon}(\tilde{r},\omega)\hspace{3pt}\tilde{r}d\tilde{r}+\int_{1+\delta}^{\infty}P_{\tau}^{\epsilon}(\tilde{r},\omega)\hspace{3pt}\tilde{r}d\tilde{r}\right\} d\omega=0.
\end{equation}
 \end{corollary}
\begin{proof}
Let $M=\lfloor\frac{2\pi}{\Delta}\rfloor$. Define $\omega_{i+1}\equiv\omega_{i}+\Delta\bmod2\pi$
for $0\leq i\leq M-1$. Then, we have from Corollary~\ref{ZeroCorollaryEntireInt},
\begin{equation}
\lim_{\tau\rightarrow0}\left[\sum_{i=0}^{M-1}\int_{\omega_{i}}^{\omega_{i+1}}\mathcal{Q}(\omega)d\omega+\int_{\omega_{i+1}}^{\omega_{0}+2\pi}\mathcal{Q}(\omega)d\omega\right]=0,\label{eq:integralsplit}
\end{equation}
 where 
\begin{equation}
\mathcal{Q}(\omega)=\int_{0}^{1-\delta}P_{\tau}^{\epsilon}(\tilde{r},\omega)\hspace{3pt}\tilde{r}d\tilde{r}+\int_{1+\delta}^{\infty}P_{\tau}^{\epsilon}(\tilde{r},\omega)\hspace{3pt}\tilde{r}d\tilde{r}.
\end{equation}
 Since $P_{\tau}^{\epsilon}(\tilde{r},\omega)\tilde{r}\geq0$, it
follows that $\mathcal{Q}(\omega)$ and both integrals in Equation~\ref{eq:integralsplit}
are \emph{non-negative} and hence each integral converges to zero \emph{independently}
giving us the desired result. 

\end{proof}
\noindent From Theorem~\ref{IntegralLimitTheorem} and Corollaries~\ref{DensityCorollary}
and \ref{ZeroCorollaryFiniteInt}, the subsequent results follow almost
immediately. 
\begin{proposition}
\label{IntegrationOverrtilde} For any given $0<\epsilon<\frac{1}{2}$,
$\omega_{0}\in[0,2\pi)$ and $0<\Delta<2\pi$, 
\begin{equation}
\lim_{\tau\rightarrow0}\int_{\omega_{0}}^{\omega_{0}+\Delta}\left\{ \int_{0}^{\infty}P_{\tau}^{\epsilon}(\tilde{r},\omega)\tilde{r}\hspace{3pt}d\tilde{r}\right\} d\omega=\int_{\omega_{0}}^{\omega_{0}+\Delta}P(\omega)d\omega.
\end{equation}
 \end{proposition}
\begin{corollary}
For any given $\omega_{0}\in[0,2\pi)$, 
\begin{equation}
\lim_{\epsilon\rightarrow0}\lim_{\Delta\rightarrow0}\frac{1}{\Delta}\lim_{\tau\rightarrow0}\int_{\omega_{0}}^{\omega_{0}+\Delta}\left\{ \int_{0}^{\infty}P_{\tau}^{\epsilon}(\tilde{r},\omega)\tilde{r}\hspace{3pt}d\tilde{r}\right\} d\omega=P(\omega_{0}).
\end{equation}

\end{corollary}

\section{Significance of our result and concluding remarks}
\noindent The integrals 
\begin{equation}
\int_{\omega_{0}}^{\omega_{0}+\Delta}\int_{1-\delta}^{1+\delta}P_{\tau}^{\epsilon}(\tilde{r},\omega)\tilde{r}d\tilde{r}d\omega,\hspace{10pt}\int_{\omega_{0}}^{\omega_{0}+\Delta}P(\omega)d\omega
\end{equation}
give the interval measures of the density functions $P_{\tau}^{\epsilon}$
(when polled close to the unit circle $\tilde{r}=1$) and $P$ respectively.
Theorem~\ref{IntegralLimitTheorem} states that at small values of
$\tau$, both the interval measures are approximately equal, with
the difference between them being $o(1)$. Furthermore the result
is also true as $\epsilon\rightarrow0$. Recall that by definition,
$P_{\tau}^{\epsilon}$ is the normalized power spectrum of the wave
function $\phi(x,y)=\exp\left(\frac{iS(x,y)}{\tau}\right)$. Hence,
we conclude that the power spectrum of $\phi(x,y)$ when polled close
to the unit circle $\tilde{r}=1$ (as $\delta\rightarrow0$ in Theorem~\ref{IntegralLimitTheorem}),
or when integrated over $\tilde{r}$ (with reference to Proposition~\ref{IntegrationOverrtilde}),
can potentially serve as a \emph{density estimator} of the orientation
of $\nabla S$ for small values of $\tau$ and $\epsilon$. Our work
is essentially an application of the higher-order stationary phase
approximation culminating in a new density estimator. 

\subsection{Advantages of our formulation}
One of the foremost advantages of our method is that the orientation
gradient density is computed without actually determining the distance
transform gradients. Since the stationary points (as seen from the stationary
phase approximation)
capture gradient information and slot them into the corresponding
frequency bins, we can directly work with the distance function---circumventing
the need to compute its derivatives. We are \emph{not aware of any
previous work} that estimates the orientation gradient density without
first computing the gradients of the distance transform.

Recall that we furnished a closed-form expression for the
distance transform gradient density function $P(\theta)$ in
Equation~\ref{def:densityfunc}.  
While it initially appears attractive, computing the density function via the
closed-form expression is practically cumbersome as we need to first determine
the Voronoi region corresponding to each Voronoi center (source point)
$Y_{k}$ and then for each orientation direction $\theta$, compute
the ray length $R_{k}(\theta)$ from $Y_{k}$ to its Voronoi boundary
along $\theta$. These involve unwieldy manipulations of complex data structures.
On the other hand, our mathematical result
provides an easy mechanism to achieve the same task as it is computationally
faster and easier to implement. Given the $N$ sampled values $\hat{S}$
of the distance function $S$ from a point-set of cardinality $K$,
we just need to compute the fast Fourier transform of
$\exp\left(\frac{i\hat{S}(x)}{\tau}\right)$---an $O(N\log N)$
operation---and then subsequently compute 
the squared magnitude (to obtain the power spectrum)---performed in
$O(N)$. Hence the orientation density function can be 
determined in $O(N\log N)$ independent of the cardinality of the point-set
($K$). Our algorithm is computationally efficient even when $K=O(N)$.

\subsection{Possible Extensions}
The present work only deals with special kinds of distance functions, namely
those defined from a set of discrete point locations in two dimensions. Other
cases include signed distance functions, distance functions defined from a set
of curves \cite{Osher88,Sethi12} etc. While our work initially appears to be 
somewhat restrictive, we should note that as the cardinality and locations in each
point-set can be arbitrary, the resulting distance functions can be quite
complex. We strongly believe that it is possible to establish a similar
Fourier transform-based density estimation result even for \emph{arbitrary}, continuous (and differentiable)
functions in 2D. A general result of this nature would subsume our current work on point-set
distance functions as well as the other kinds of distance functions mentioned above,
as the gradient magnitude of an arbitrary 2D function can vary across the
point locations and need not necessarily be identically equal to one as in the
case of distance functions. Here the gradient density function will inherently
be two dimensional and defined both along the radial, gradient magnitude
direction $\left(r = \sqrt{S_x^2+S_y^2}\right)$ and the orientation
$\left(\theta = \arctan\left(\frac{S_y}{S_x}\right)\right)$. However, at the
present time, due to many technical issues, 
this is merely a conjecture and requires further investigation. Generalization
of the present work 
to three dimensions is also a concrete possibility. These are fruitful
avenues for future research and we may explore them in the years to
come. 

\appendix
\section{\label{sec:proofIntLemma}Proof of Lemma~\ref{IntegralLemma}}
\begin{proof}
Define a function $H(x,y)$ by 
\[
H(x,y)\equiv\left\{ \begin{array}{ll}
1: & \mbox{if }(x,y)\in\Omega^{\epsilon};\\
0: & \mbox{otherwise}.
\end{array}\right.
\]
 Let $f(x,y)=H(x,y)\exp\left(\frac{iS(x,y)}{\tau}\right)$. Then,
\begin{equation}
F_{\tau}^{\epsilon}(u,v)=\frac{1}{2\pi\tau l^{\epsilon}}\iint f(x,y)\exp\left(\frac{-i(ux+vy)}{\tau}\right)dxdy.
\end{equation}
 Let $\frac{u}{\tau}=s$, $\frac{v}{\tau}=t$ and $G(s,t)=F_{\tau}^{\epsilon}(s\tau,t\tau)$.
Then, 
\begin{equation}
\tau l^{\epsilon}G(s,t)=\frac{1}{2\pi}\iint f(x,y)\exp\left(-i(sx+ty)\right)dxdy.
\end{equation}
 Since $f$ is $\ell^{1}$ integrable, by Parseval's theorem \cite{Bracewell99},
we have
\begin{equation}
\iint\left|f(x,y)\right|^{2}dxdy=\iint\left|\tau l^{\epsilon}G(s,t)\right|^{2}dsdt=(\tau l^{\epsilon})^{2}\iint\left|F_{\tau}^{\epsilon}(s\tau,t\tau)\right|^{2}dsdt.
\end{equation}
 Letting $u=s\tau$, $v=t\tau$ and observing that 
\begin{equation}
\iint\left|f(x,y)\right|^{2}dxdy=
\iint\limits _{\Omega^{\epsilon}}\left|\exp\left(\frac{iS(x,y)}{\tau}\right)\right|^{2}dxdy=L^{\epsilon},
\end{equation}
 we get 
\begin{equation}
(l^{\epsilon})^{2}\iint\left|F_{\tau}^{\epsilon}(u,v)\right|^{2}dudv=L^{\epsilon}.
\end{equation}
 Hence 
\begin{equation}
\iint\left|F_{\tau}^{\epsilon}(u,v)\right|^{2}dudv=1
\end{equation}
 which completes the proof. 

\end{proof}

\section{\label{sec:pitfalls5Dstationaryphase}Difficulty with the 5D stationary
phase approximation}

Since $P_{\tau}^{\epsilon}(\tilde{r},\omega)$ equals $F_{\tau}^{\epsilon}(\tilde{r},\omega)\overline{F_{\tau}^{\epsilon}(\tilde{r},\omega)}$,
we have 
\begin{equation}
I(\omega)=\sum_{j=1}^{K}\sum_{k=1}^{K}\frac{1}{(2\pi\tau l^{\epsilon})^{2}}N_{jk}(\omega),
\end{equation}
 where 
\begin{equation}
N_{jk}(\omega)=\int_{1-\delta}^{1+\delta}\int_{0}^{2\pi}\int_{R_{k}^{(1)}(\theta^{\prime})}^{R_{k}^{(2)}(\theta^{\prime})}\int_{0}^{2\pi}\int_{R_{k}^{(1)}(\theta)}^{R_{k}^{(2)}(\theta)}\exp\left(\frac{i}{\tau}b_{jk}\right)f_{1}drd\theta dr^{\prime}d\theta^{\prime}d\tilde{r}.
\end{equation}
 Here, 
\begin{eqnarray}
b_{jk}(r,\theta,r^{\prime},\theta^{\prime},\tilde{r};\omega) & = & r\left[1-\tilde{r}\cos(\theta-\omega)\right]-r^{\prime}\left[1-\tilde{r}\cos(\theta^{\prime}-\omega)\right]\nonumber \\
 &  & -\tilde{r}\left[\cos(\omega)(x_{j}-x_{k})+\sin(\omega)(y_{j}-y_{k})\right]
\end{eqnarray}
 and 
\begin{equation}
f_{1}(r,r^{\prime},\tilde{r})=rr^{\prime}\tilde{r}.
\end{equation}
 Notice that the phase term of the quantity $C_{j}\overline{C_{k}}$,
namely 
\begin{equation}
-\tilde{r}\left[\cos(\omega)(x_{j}-x_{k})+\sin(\omega)(y_{j}-y_{k})\right]\label{eq:phasetermCjCk}
\end{equation}
 is absorbed in $b_{jk}$. Since we are interested only in the limit
as $\tau\rightarrow0$, essential contribution to $N_{jk}(\omega)$
comes only from the stationary (critical) point(s) of $b_{jk}$ \cite{Wong89}.
The partial derivatives of $b_{jk}(r,\theta,r^{\prime},\theta^{\prime},\tilde{r})$
are given by 
\begin{eqnarray}
\frac{\partial b_{jk}}{\partial r} & = & 1-\tilde{r}\cos(\theta-\omega),\hspace{10pt}\frac{\partial b_{jk}}{\partial\theta}=r\tilde{r}\sin(\theta-\omega),\nonumber \\
\frac{\partial b_{jk}}{\partial r^{\prime}} & = &
-1+\tilde{r}\cos(\theta^{\prime}-\omega),\hspace{10pt}\frac{\partial
  b_{jk}}{\partial\theta^{\prime}}=-r^{\prime}\tilde{r}\sin(\theta^{\prime}-\omega),~\mbox{\rm
  and}\nonumber \\
\frac{\partial b_{jk}}{\partial\tilde{r}} & = & -r\cos(\theta-\omega)+r^{\prime}\cos(\theta^{\prime}-\omega)-[\cos(\omega)(x_{j}-x_{k})+\sin(\omega)(y_{j}-y_{k})].
\end{eqnarray}
 As $r$, $r^{\prime}$ and $\tilde{r}>0$, it is easy to see that
for $\nabla b_{jk}=0$ (\emph{stationary}), we must have
\begin{equation}
\tilde{r}=1,\hspace{10pt}\theta=\theta^{\prime}=\omega,r=r^{\prime}-[\cos(\omega)(x_{j}-x_{k})+\sin(\omega)(y_{j}-y_{k})].
\end{equation}
 Let $t_{0}$ denote the stationary point. The Hessian matrix $\mathcal{W}$
of $b_{jk}$ at $t_{0}$ is given by 
\[
\left.\mathcal{W}(r,\theta,r^{\prime},\theta^{\prime},\tilde{r})\right\vert_{t_{0}}=\left[\begin{array}{ccccc}
0 & 0 & 0 & 0 & -1\\
0 & r_{t_{0}} & 0 & 0 & 0\\
0 & 0 & 0 & 0 & 1\\
0 & 0 & 0 & -r^{\prime} & 0\\
-1 & 0 & 1 & 0 & 0
\end{array}\right]
\]
where $r_{t_{0}}=r^{\prime}-[\cos(\omega)(x_{j}-x_{k})+\sin(\omega)(y_{j}-y_{k})]$.
Unfortunately, the determinant of $\mathcal{W}$ at the stationary
point $t_{0}$ equals $0$ as the first and third rows---corresponding
to $r$ and $r^{\prime}$ respectively---are scalar multiples of each
other.  This impedes us from directly applying the
5D stationary phase approximation \cite{Wong89}. 

The addition of a $6^{\rm th}$ integral to the above setup---where the power spectrum
is integrated over a small range on the orientation $\omega$ (in order
to remove cross phase factors)---leads to a 6D
stationary phase approximation. This is of no help either, as the Hessian
continues to remain degenerate for the same reasons as above. 

\subsection{Avoiding degeneracy by symmetry breaking}
\label{sec:symmetrybreaking}
As we notice above, the degeneracy in the 5D (and 6D) stationary phase approximation
arises because the determinant of the Hessian, namely $\mathcal{W}$, when
evaluated at the stationary point $t_0$ takes the value zero, as its first and
third rows corresponding to $r$ and $r^{\prime}$ respectively are scalar
multiples of each other. Also, observe that the value of either $r$ or
$r^{\prime}$ is not determined at the stationary point and can take on
arbitrary values. However, the rows (and columns) of $\mathcal{W}$
corresponding to the other three variables $\theta, \theta^{\prime}$ and
$\tilde{r}$ are indeed independent of each other and do not cause
degeneracy. This strongly suggests that if we \emph{do not} consider both $r$
and $r^{\prime}$ together and hold back either one of them, say $r^{\prime}$,
the resulting 4D stationary phase approximation will be well-defined. Since
the integration range for $r^{\prime}$ is defined in terms of $\theta^{\prime}$,
we retain both these variables and perform the stationary phase approximation
on the other three variables. This manual breaking of symmetry avoids the degeneracy
issue. 

\section{\label{sec:prooflemmagjk} Proof of Lemma~\ref{lemma:gjk}}
\begin{proof}
Recall that the essential contribution to $g_{jk}(r^{\prime},\theta^{\prime};\omega)$
comes only from the stationary points of $\gamma_{jk}$ as $\tau\rightarrow0$
\cite{Wong89}. The partial derivatives of $\gamma_{jk}(r,\theta,\tilde{r};r^{\prime},\theta^{\prime},\omega)$
are given by 
\begin{eqnarray}
\frac{\partial\gamma_{jk}}{\partial r} & = &
1-\tilde{r}\cos(\theta-\omega),\hspace{10pt}\frac{\partial\gamma_{jk}}{\partial\theta}=r\tilde{r}\sin(\theta-\omega),
~\mbox{\rm and}
\nonumber\\ 
\frac{\partial\gamma_{jk}}{\partial\tilde{r}} & = &
-r\cos(\theta-\omega)+r^{\prime}\cos(\theta^{\prime}-\omega)-[\cos(\omega)(x_{j}-x_{k})+\sin(\omega)(y_{j}-y_{k})]. 
\end{eqnarray}
 As both $r$ and $\tilde{r}>0$, for $\nabla\gamma_{jk}=0$
(\emph{stationary}), we must have 
\begin{eqnarray}
\tilde{r} & = & 1,\hspace{5pt}\theta=\omega,\hspace{5pt}\mbox{and}\nonumber \\
r & = & r^{\prime}\cos(\theta^{\prime}-\omega)-[\cos(\omega)(x_{j}-x_{k})+\sin(\omega)(y_{j}-y_{k})].
\end{eqnarray}
 Let $t_{0}$ denote a stationary point. Then 
\begin{eqnarray*}
\gamma_{jk}(t_{0}) & = & r^{\prime}\cos(\theta^{\prime}-\omega)-[\cos(\omega)(x_{j}-x_{k})+\sin(\omega)(y_{j}-y_{k})]=r_{jk}(r^{\prime},\theta^{\prime};\omega),\\
f_{2}(t_{0}) & = & r_{jk}(r^{\prime},\theta^{\prime};\omega)
\end{eqnarray*}
 and the Hessian matrix $\mathcal{H}$ of $\gamma_{jk}$ at the stationary
point $t_{0}$ is 
\[
\left.\mathcal{H}(r,\theta,\tilde{r})\right\vert_{t_{0}}=\left[\begin{array}{ccc}
0 & 0 & -1\\
0 & r_{jk}(r^{\prime},\theta^{\prime};\omega) & 0\\
-1 & 0 & 0
\end{array}\right].
\]
 It can be easily verified that the determinant of $\mathcal{H}$
equals $-r_{jk}(r^{\prime},\theta^{\prime})$.

If $r_{jk}(r^{\prime},\theta^{\prime};\omega)\leq0$, no stationary points
exist as $r>0$ by definition and hence $g_{jk}(r^{\prime},\theta^{\prime};\omega)=0$
as $\tau\rightarrow0$ \cite{Wong89}. If $r_{jk}(r^{\prime},\theta^{\prime};\omega)>0$,
the determinant of $\mathcal{H}$ is strictly negative and its
signature--the difference
between the number of positive and negative eigenvalues--is 1. Then,
from the higher-order stationary phase approximation \cite{Wong89},
we have 
\begin{eqnarray*}
g_{jk}(r^{\prime},\theta^{\prime};\omega) & = & (2\pi\tau)^{\frac{3}{2}}\sqrt{r_{jk}(r^{\prime},\theta^{\prime};\omega)}\exp\left(\frac{ir_{jk}(r^{\prime},\theta^{\prime};\omega)}{\tau}+\frac{i\pi}{4}\right)+\epsilon_{1}(r^{\prime},\theta^{\prime},\tau;\omega)
\end{eqnarray*}
 as $\tau\rightarrow0$, where $\epsilon_{1}(r^{\prime},\theta^{\prime},\tau;\omega)$
includes the contributions from the boundary in Equation~\ref{def:g}.
Here we have assumed that the stationary point $t_{0}$ does not occur
on the boundary and lies to its interior, i.e, $R_{j}^{(1)}(\theta)<r_{jk}(r^{\prime},\theta^{\prime};\omega)<R_{j}^{(2)}(\theta)$,
as the measure on the set of $\{\omega,\theta^{\prime},r^{\prime}\}$
for which $r_{jk}(r^{\prime},\theta^{\prime};\omega)$ (or $t_{0}$)
can occur on the boundary is zero.

Let $\Gamma$ denote the boundary in Equation~\ref{def:g}. If there
does not exist a 2D patch on $\Gamma$ on which $\gamma_{jk}$ is
constant, then we can conclude that $\epsilon_{1}(r^{\prime},\theta^{\prime},\tau;\omega)$---which
includes the contributions from the boundary $\Gamma$ involving the stationary
points of the second kind where the level curves of $\gamma_{jk}$
are tangential to $\Gamma$---should be at least $O(\tau^{2})$ as
$\tau\rightarrow0$ \cite{Jones58,Cooke82,Wong89}. From this, we
get 
\begin{equation}
\epsilon_{1}(r^{\prime},\theta^{\prime},\tau;\omega)=\tau^{\kappa}\xi_{jk}(r^{\prime},\theta^{\prime};\omega)
\end{equation}
 where $\kappa\geq2$ and $\xi_{jk}(r^{\prime},\theta^{\prime};\omega)$
is some bounded, continuous function. Since the boundary $\Gamma$
is made of straight line segments, we can show that this is indeed
the case. Below, we take care of this technical issue.

The boundary $\Gamma$ in Equation~\ref{def:g} is the \emph{union}
of two disconnected surfaces $\Gamma_{1}=\mathcal{A}_{1}\times[1-\delta,1+\delta]$
and $\Gamma_{2}=\mathcal{A}_{2}\times[1-\delta,1+\delta]$ where $\mathcal{A}_{1}$
is the boundary along $r(\theta)=R_{j}^{(1)}(\theta)$ and $\mathcal{A}_{2}$
is the boundary along $r(\theta)=R_{j}^{(2)}(\theta)$. Note that
both $\mathcal{A}_{1}$ and $\mathcal{A}_{2}$ are composed of a finite
sequence of straight line segments. Consider the surface $\Gamma_{1}$.
The value of $\gamma_{jk}$ on the surface $\Gamma_{1}$ at a given
$\theta$ and $\tilde{r}$ (with $r^{\prime}$, $\theta^{\prime}$
and $\omega$ held fixed) equals 
\begin{equation}
\gamma_{jk}^{\Gamma_{1}}(\theta,\tilde{r};r^{\prime},\theta^{\prime},\omega)=R_{j}^{(1)}(\theta)[1-\tilde{r}\cos(\theta-\omega)]+\tilde{r}r_{jk}(r^{\prime},\theta^{\prime};\omega).
\end{equation}
 Following the lines of Theorem~\ref{CircleTheorem}, we observe
that for a given $\tilde{r}$, $\gamma_{jk}^{\Gamma_{1}}(\theta,\tilde{r};r^{\prime},\theta^{\prime},\omega)$
\emph{cannot} be constant for a contiguous interval of $\theta$ as
Equation~\ref{eq:sincoslinearrelation} cannot be satisfied over
any finite interval. By a similar argument, there can exist at most
only a \emph{finite} discrete set of $\theta$ for which $R_{j}^{(1)}(\theta)\cos(\theta-\omega)=r_{jk}(r^{\prime},\theta^{\prime};\omega)$.
Let $\mathcal{Z}$ denote this finite set. Then, for a given $\theta\notin\mathcal{Z}$,
$\gamma_{jk}^{\Gamma_{1}}$ varies linearly in $\tilde{r}$ and specifically,
its derivative with respect to $\tilde{r}$ does not vanish. From the
above observations, we can conclude that there does not exist a 2D
patch on $\Gamma_{1}$ on which $\gamma_{jk}^{\Gamma_{1}}$ is constant.
A similar conclusion can be obtained even for the surface $\Gamma_{2}$.
Hence, $\gamma_{jk}$ \emph{cannot} be constant on the boundary $\Gamma$
over a 2D region having a finite non-zero measure. 

\end{proof}

\section{\label{sec:Ijk1equalsJjk1} Proof of Lemma~\ref{lemma:Ijk1equalsJjk1}}
\begin{proof}
By construction, the integrals $J_{jk}^{(2)}(\beta,\omega)$ and $J_{jk}^{(3)}(\beta,\omega)$
\emph{do not} include the stationary point $\theta^{\prime}=\omega$
and hence $\nabla p\not=0$ in these integrals. Following
the lines of Theorem~\ref{CircleTheorem}, by defining the vector
field $\mathbf{u}=\frac{\nabla p}{\|\nabla p\|^{2}}q$
and then applying the divergence theorem, both $J_{jk}^{(2)}(\beta,\omega)$
and $J_{jk}^{(3)}(\beta,\omega)$ can be shown to be $\tau^{\kappa_{2}}\zeta^{(2)}(\beta,\omega)$
and $\tau^{\kappa_{3}}\zeta^{(3)}(\beta,\omega)$ respectively where
both $\kappa_{2}$ and $\kappa_{3}\geq0.5$ and $\zeta^{(2)}$ and
$\zeta^{(3)}$ are some continuous bounded functions of $\beta$ and
$\omega$. Hence, we can conclude that 
\begin{equation}
\left|\lim_{\tau\rightarrow0}\int_{0}^{2\pi}\frac{\eta_{jk}}{L^{\epsilon}}J_{jk}^{(2)}(\beta,\omega)d\omega\right|\leq\lim_{\tau\rightarrow0}\frac{\tau^{\kappa_{2}}}{L^{\epsilon}}\int_{0}^{2\pi}\left|\zeta^{(2)}(\beta,\omega)\right|d\omega=0
\end{equation}
 as $|\eta_{jk}=1|$ and similarly for $J_{jk}^{(3)}(\beta,\omega)$
for $any$ fixed $\beta>0$. It follows that the result also holds
as $\beta\rightarrow0$ provided the limit for $\beta$ is considered
after the limit for $\tau$, i.e, 
\begin{eqnarray}
\lim_{\beta\rightarrow0}\lim_{\tau\rightarrow0}\int_{\omega_{0}}^{\omega_{0}+\Delta}\frac{\eta_{jk}}{L^{\epsilon}}J_{jk}^{(2)}(\beta,\omega)d\omega & = & 0,\,\mathrm{and}\nonumber \\
\lim_{\beta\rightarrow0}\lim_{\tau\rightarrow0}\int_{\omega_{0}}^{\omega_{0}+\Delta}\frac{\eta_{jk}}{L^{\epsilon}}J_{jk}^{(3)}(\beta,\omega)d\omega & = & 0.
\end{eqnarray}

\noindent Hence, $I_{jk}^{(1)}(\omega)$ in Equation~\ref{eq:I1jk}
can be approximated by $J_{jk}^{(1)}(\beta,\omega)$ as $\beta\rightarrow0$
and as $\tau\rightarrow0$. 

\end{proof}

\section{\label{sec:psijkequalsP} Proof of Lemma~\ref{lemma:psijkequalsP}}
\begin{proof}
Define 
\begin{equation}
\rho_{jk}(\beta,\omega)=\int_{r_{k}^{(-)}(\beta,\omega)}^{r_{k}^{(+)}(\beta,\omega)}\sqrt{r^{\prime}}\sqrt{r^{\prime}-\alpha_{jk}(\omega)}dr^{\prime}.
\end{equation}
 We consider two cases, one in which $j=k$ and another in which
$j\not=k$.

case(i): If $j\not=k$, then $\alpha_{jk}(\omega)$ varies continuously
with $\omega$. Also, notice that $\rho_{jk}(\beta,\omega)$ is \emph{independent}
of $\tau$ and is also a bounded function of $\beta$ and $\omega$.
The stationary point(s) of $\alpha_{jk}$---denoted by $\tilde{\omega}$---satisfy
\begin{equation}
\tan(\tilde{\omega})=\frac{y_{j}-y_{k}}{x_{j}-x_{k}},\label{eq:alphastationarypt}
\end{equation}
 and the second derivative of $\alpha_{jk}(\omega)$ at its stationary
point(s) is given by 
\begin{equation}
\alpha_{jk}^{\prime\prime}(\tilde{\omega})=-\alpha_{jk}(\tilde{\omega}).
\end{equation}
 For $\alpha_{jk}^{\prime\prime}(\tilde{\omega})=0$, we must have
\begin{equation}
\tan(\tilde{\omega})=-\frac{x_{j}-x_{k}}{y_{j}-y_{k}}=\frac{y_{j}-y_{k}}{x_{j}-x_{k}},
\end{equation}
 where the last equality is obtained using Equation~\ref{eq:alphastationarypt}.
Rewriting, we get 
\begin{equation}
\left(\frac{y_{j}-y_{k}}{x_{j}-x_{k}}\right)^{2}=-1
\end{equation}
 which cannot be true. Since the second derivative cannot vanish at
the stationary point $\tilde{\omega}$, from the one-dimensional stationary
phase approximation \cite{OlverBook74}, we have 
\begin{equation}
\lim_{\tau\rightarrow0}\frac{1}{L^{\epsilon}}\int_{\omega_{0}}^{\omega_{0}+\Delta}\exp\left(\frac{-i\alpha_{jk}(\omega)}{\tau}\right)\rho_{jk}(\beta,\omega)d\omega=\lim_{\tau\rightarrow0}O(\tau^{\kappa})=0
\end{equation}
 where $\kappa=0.5$ or 1 depending upon whether the interval $[\omega_{0},\omega_{0}+\Delta)$
contains the stationary point ($\tilde{\omega}$) or not. Hence, we
have $\psi_{jk}(\beta)=0$ for $j\not=k$.

case(ii): If $j=k$, then $\alpha_{kk}(\omega)=0$ and 
\begin{eqnarray}
\rho_{kk}(\beta,\omega) & = & \int_{r_{k}^{(-)}(\beta,\omega)}^{r_{k}^{(+)}(\beta,\omega)}r^{\prime}dr^{\prime},\nonumber \\
\psi_{kk}(\beta) & = & \frac{1}{{L^{\epsilon}}}\int_{\omega_{0}}^{\omega_{0}+\Delta}\rho_{kk}(\beta,\omega)d\omega.
\end{eqnarray}
 From the definitions of $r_{k}^{(1)}(\beta,\omega)$ and $r_{k}^{(2)}(\beta,\omega)$
in Equation~\ref{eq:rk1andrk2}, we observe that 
\begin{eqnarray}
\lim_{\beta\rightarrow0}r_{k}^{(1)}(\beta,\omega) & \uparrow &
R_{k}^{(1)}(\omega),~\mbox{\rm and}\nonumber \\
\lim_{\beta\rightarrow0}r_{k}^{(2)}(\beta,\omega) & \downarrow & R_{k}^{(2)}(\omega).
\end{eqnarray}
 Since $r_{k}^{(-)}(\beta,\omega)\rightarrow r_{k}^{(1)}(\beta,\omega)$
and $r_{k}^{(+)}(\beta,\omega)\rightarrow r_{k}^{(2)}(\beta,\omega)$
as $\beta\rightarrow0$, we have 
\begin{eqnarray}
\lim_{\beta\rightarrow0}r_{k}^{(-)}(\beta,\omega) & = & R_{k}^{(1)}(\omega),\mbox{ and }\nonumber \\
\lim_{\beta\rightarrow0}r_{k}^{(+)}(\beta,\omega) & = & R_{k}^{(2)}(\omega).
\end{eqnarray}
 Since $r_{k}^{(-)}(\beta,\omega)\geq r_{k}^{(1)}(\beta,\omega)$
and $r_{k}^{(+)}(\beta,\omega)\leq r_{k}^{(2)}(\beta,\omega)$ at
a fixed $\beta$ and $r^{\prime}>0$, we see that $\rho_{kk}(\beta,\omega)$
can be bounded from above by a positive \emph{decreasing} function
of $\beta$, namely,
\begin{equation}
\rho_{kk}(\beta,\omega)\leq\int_{r_{k}^{(1)}(\beta,\omega)}^{r_{k}^{(2)}(\beta,\omega)}r^{\prime}dr^{\prime},
\end{equation}
 and is also independent of $\tau$. As both $r_{k}^{(1)}(\beta,\omega)$
and $r_{k}^{(2)}(\beta,\omega)$ are also bounded functions, by the
Lebesgue dominated convergence theorem, we get 
\begin{eqnarray}
\lim_{\beta\rightarrow0}\lim_{\tau\rightarrow0}\psi_{kk}(\beta) & = & \frac{1}{L^{\epsilon}}\int_{\omega_{0}}^{\omega_{0}+\Delta}\lim_{\beta\rightarrow0}\rho_{kk}(\beta,\omega)d\omega\nonumber \\
 & = & \frac{1}{L^{\epsilon}}\int_{\omega_{0}}^{\omega_{0}+\Delta}\left\{ \int_{R_{k}^{(1)}(\omega)}^{R_{k}^{(2)}(\omega)}r^{\prime}dr^{\prime}\right\} d\omega\nonumber \\
 & = & \frac{(1-2\epsilon)}{L^{\epsilon}}\int_{\omega_{0}}^{\omega_{0}+\Delta}\frac{R_{k}^{2}(\omega)}{2}d\omega.
\end{eqnarray}
 Recall that $L^{\epsilon}=(1-2\epsilon)L$. Hence, 
\begin{eqnarray}
\sum_{j=1}^{K}\sum_{k=1}^{K}\lim_{\beta\rightarrow0}\lim_{\tau\rightarrow0}\psi_{jk}(\beta) & = & \frac{1}{L}\sum_{k=1}^{K}\int_{\omega_{0}}^{\omega_{0}+\Delta}\frac{R_{k}^{2}(\omega)}{2}d\omega\nonumber \\
 & = & \int_{\omega_{0}}^{\omega_{0}+\Delta}P(\omega)d\omega
\end{eqnarray}
 which completes the proof. 

\end{proof}
\noindent \bibliographystyle{siam}
\bibliography{DisTransDensityEstimation}

\end{document}